\documentclass[nohyperref]{article}

\usepackage{hyperref}

\usepackage{microtype}
\usepackage{graphicx}
\usepackage{booktabs} %

\usepackage{amsmath, amsfonts}
\usepackage{amsthm}
\usepackage{bm}
\usepackage{color}
\usepackage{amssymb}
\usepackage{mathtools}

\usepackage{subcaption}
\usepackage{xcolor}
\usepackage{xr}
\usepackage{navigator}
\usepackage{url}
\usepackage{soul}
\usepackage{balance}

\usepackage[accepted]{icml2022}

\theoremstyle{plain}
\newtheorem{theorem}{Theorem}[section]

\theoremstyle{definition}
\newtheorem{definition}[theorem]{Definition}
\theoremstyle{remark}
\newtheorem{remark}[theorem]{Remark}
\newtheorem{innerlemma}{Lemma}
\newenvironment{lemma}[1]
  {\renewcommand\theinnerlemma{#1}\innerlemma}
    {\endinnerlemma}
\newtheorem{innerassumption}{Assumption}
\newenvironment{assumption}[1]
  {\renewcommand\theinnerassumption{#1}\innerassumption}
    {\endinnerassumption}

\graphicspath{ {./} }
\usepackage{thmtools}
\usepackage{thm-restate}
\usepackage{dcolumn}
\usepackage{multirow}
\usepackage{enumitem}

\newcolumntype{d}[1]{D{.}{.}{4}}%
\newcommand{\subhead}[1]{\multicolumn{1}{c}{#1}}%

\DeclareMathOperator{\y}{\textbf{y}}

\DeclareMathOperator{\B}{\textbf{B}}

\DeclareMathOperator{\X}{\textbf{X}}

\DeclareMathOperator{\G}{\textbf{G}}
\DeclareMathOperator{\A}{\textbf{A}}

\DeclareMathOperator{\zero}{\textbf{0}}

\newcommand{\lrVert}[1]{\left\Vert #1 \right\Vert}
\newcommand{\Ebatch}[1]{\mathbb{E}^{t_0} \left[#1\right]}
\newcommand{\Etot}[1]{\mathbb{E} \left[#1\right]}

\newcommand{\rev}[1]{{#1}}

\icmltitlerunning{Communication-Efficient Learning with Vertically Partitioned Data}

\begin{document}

\twocolumn[
\icmltitle{Compressed-VFL: Communication-Efficient Learning with Vertically Partitioned Data}

\begin{icmlauthorlist}
\icmlauthor{Timothy Castiglia}{rpi}
\icmlauthor{Anirban Das}{rpi}
\icmlauthor{Shiqiang Wang}{ibm}
\icmlauthor{Stacy Patterson}{rpi}
\end{icmlauthorlist}

\icmlaffiliation{rpi}{Department of Computer Science,
  Rensselaer Polytechnic Institute, Troy, NY, USA}
\icmlaffiliation{ibm}{IBM Thomas J. Watson Research Center, Yorktown Heights, NY, USA}

\icmlcorrespondingauthor{Timothy Castiglia}{castit@rpi.edu}
\icmlkeywords{Machine Learning, Federated Learning, Distributed Systems, Parellel Computing, Scalable Algorithms, Compression}

\vskip 0.3in
]
\printAffiliationsAndNotice{} %

\begin{abstract}
We propose Compressed Vertical Federated Learning (C-VFL) for communication-efficient training on vertically partitioned data. In C-VFL, a server and multiple parties collaboratively train a model on their respective features utilizing several local iterations and sharing compressed intermediate results periodically. Our work provides the first theoretical analysis of the effect message compression has on distributed training over vertically partitioned data. We prove convergence of non-convex objectives at a rate of $O(\frac{1}{\sqrt{T}})$ when the compression error is bounded over the course of training. We provide specific requirements for convergence with common compression techniques, such as quantization and top-$k$ sparsification. Finally, we experimentally show compression can reduce communication by over $90\%$ without a significant decrease in accuracy over VFL without compression.
\end{abstract}

\section{Introduction} \label{intro.sec}

Federated Learning~\citep{pmlr-v54-mcmahan17a} 
is a distributed machine learning approach
that has become of much interest
in both theory~\citep{li2018federated, wang2019adaptive, liu2020client} 
and practice~\citep{DBLP:conf/mlsys/BonawitzEGHIIKK19, DigitalHealthFL, DBLP:journals/comsur/LimLHJLYNM20} 
in recent years. 
Naive distributed learning algorithms may require frequent exchanges of large amounts of data, which can lead to slow training performance~\citep{lin2018don}.
Further, participants may be globally distributed, with high latency network connections. To mitigate these factors, 
Federated Learning algorithms aim to be communication-efficient by design.
Methods such as \emph{local updates}~\citep{moritz2016sparknet,liu2019communication}, 
where parties train local parameters for multiple iterations without communication,
and message compression~\citep{DBLP:conf/nips/StichCJ18, DBLP:conf/nips/WenXYWWCL17, DBLP:conf/icml/KarimireddyRSJ19} 
reduce message frequency and size, respectively, 
with little impact on training performance. 

Federated Learning methods often
target the case where the data among parties is distributed 
horizontally: each party's data shares the same features but
parties hold data corresponding to different sample IDs. This is known as
Horizontal Federated Learning (HFL)~\citep{DBLP:journals/tist/YangLCT19}. 
However, there are several application areas where 
data is partitioned in a \emph{vertical} manner: 
the parties store data on the same sample IDs
but different feature spaces.

An example of a vertically partitioned setting includes 
a hospital, bank, and insurance company 
seeking to train a model to predict
something of mutual interest, such as customer credit score.
Each of these institutions may have data on the same
individuals but store medical history, financial
transactions, and vehicle accident reports, respectively.
These features must remain local to the institutions
due to privacy concerns, rules and regulations (e.g., GDPR, HIPAA), 
and/or communication network limitations.
In such a scenario, Vertical Federated Learning (VFL) methods
must be employed.
Although VFL is less well-studied than HFL, there has
been a growing interest in VFL 
algorithms recently~\citep{FDML, gu2021privacy, verticalAutoencoders},
and VFL algorithms have important applications
including risk prediction, smart manufacturing, and discovery
of pharmaceuticals~\citep{kairouz2019advances}.

Typically in VFL, each party trains a local embedding 
function that maps raw data features to 
a meaningful vector representation, or \emph{embedding}, for prediction tasks.
For example, a neural network can be an embedding function
for mapping the text of an online article  
to a vector space for classification~\citep{book_embeddings_git}.
Referring to Figure~\ref{vflmodel.fig}, suppose Party $1$ 
is a hospital with medical data features $x_1$. 
The hospital computes its embedding $h_1(\theta_1 ; x_1)$ for the features
by feeding $x_1$ through a neural network. 
The other parties (the bank and insurance company),
compute embeddings for their features, then all parties share the
embeddings in a private manner
(e.g., homomorphic encryption,
secure multi-party computation, or secure aggregation).
The embeddings are then combined 
in a \emph{server model} $\theta_0$ to determine the final loss of the global model.
A server model (or fusion network) captures
the complicated interactions of embeddings
and is often a complex, non-linear 
model~\citep{GuLSLCLM19, NieLWWS21, han2021improving}. 
Embeddings can be very large, in practice, sometimes requiring 
terabytes of communication over the course of training.

\begin{figure}[t]
    \centering
    \includegraphics[width=0.4\textwidth]{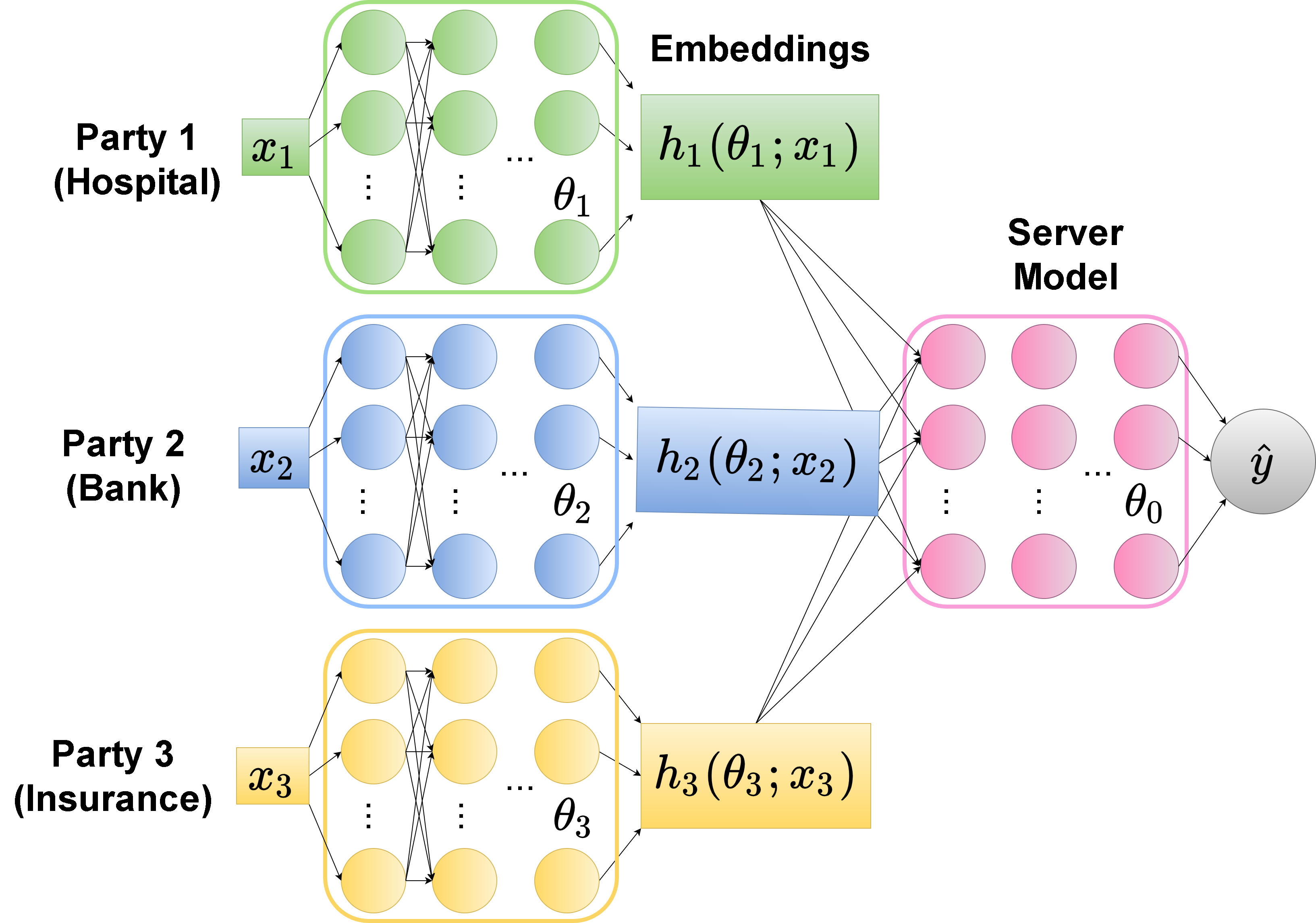}
    \caption{Example global model with neural networks.
        To obtain a $\hat{y}$ prediction for a data sample $x$, each
        party $m$ feeds the local features of $x$, $x_m$,
        into a neural network. The output of this neural network is
        the embedding $h_m(\theta_m ; x_m)$. All embeddings are then
        fed into the server model neural network with parameters $\theta_0$.
    }
    \label{vflmodel.fig}
\end{figure}

Motivated by this, we propose Compressed Vertical Federated Learning (\mbox{C-VFL}),
a general framework for communication-efficient Federated Learning 
over vertically partitioned data.
In our algorithm, 
parties communicate compressed embeddings periodically, 
and the parties and the server 
each run block-coordinate descent for multiple local iterations, in parallel,
using stochastic gradients to update their local parameters.

\mbox{C-VFL} is the first theoretically verified VFL algorithm that applies
embedding compression.
Unlike in HFL algorithms, \mbox{C-VFL} compresses embeddings rather
than gradients. 
Previous work has proven convergence for HFL algorithms with gradient 
compression~\citep{DBLP:conf/nips/StichCJ18, DBLP:conf/nips/WenXYWWCL17, DBLP:conf/icml/KarimireddyRSJ19}.
However, no previous work
analyzes the convergence requirements for VFL algorithms that
use embedding compression.
Embeddings are parameters in the partial derivatives calculated at each party.
The effect of compression error on the resulting partial derivatives may be complex; 
therefore, the analysis in previous work on gradient compression in HFL 
does not apply to compression in VFL.
In our work, we prove that, under a diminishing compression error,
\mbox{C-VFL} converges at a rate of $O(\frac{1}{\sqrt{T}})$,
which is comparable to previous VFL algorithms that do not employ compression.
We also analyze common compressors, such as quantization and sparsification,
in \mbox{C-VFL} and provide bounds on their compression parameters to ensure convergence.

\begin{figure}[t]
    \centering
    \includegraphics[width=0.35\textwidth]{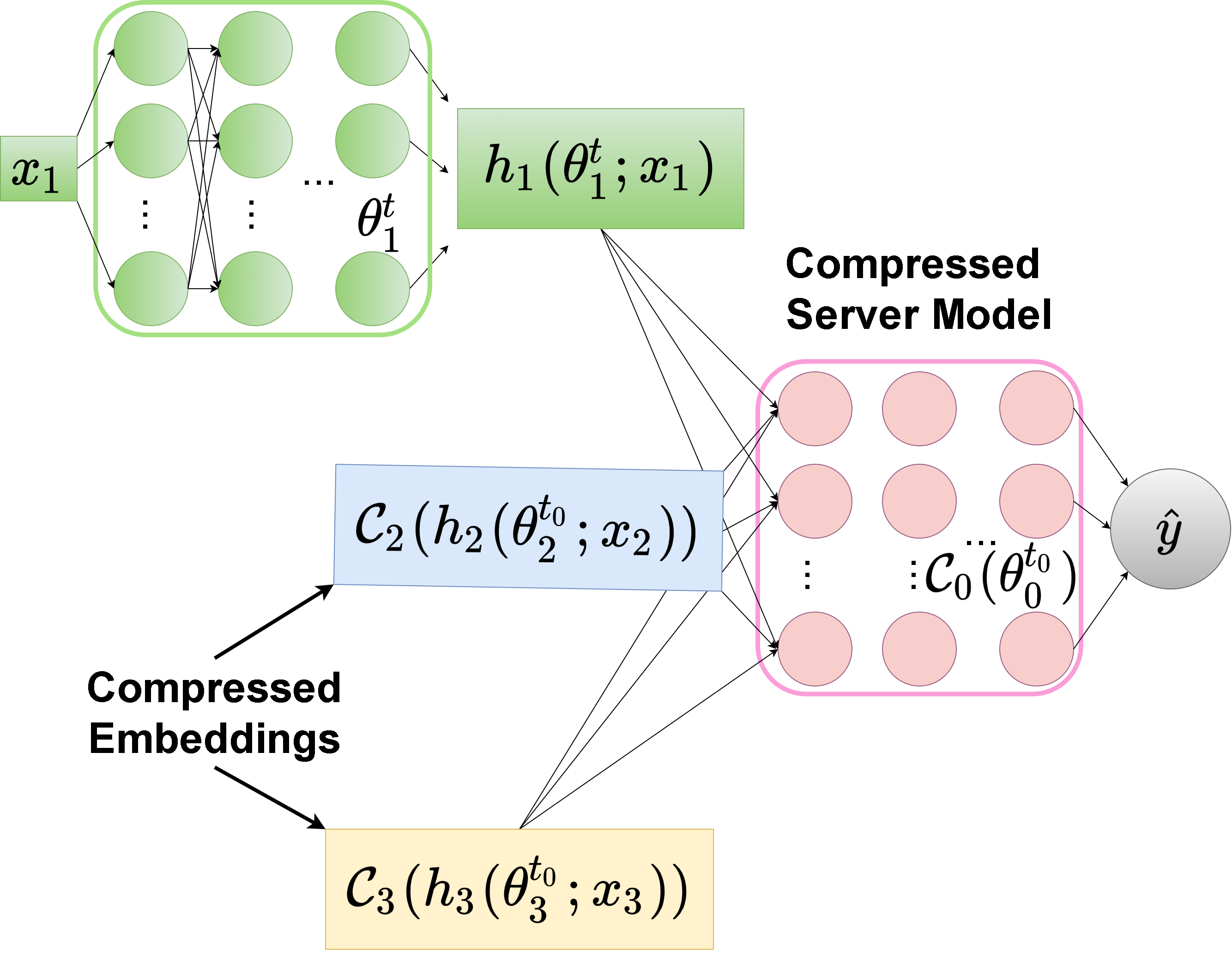}
    \caption{Example local view of a global model with neural networks.
        When running \mbox{C-VFL}, Party $1$ (in green) only 
        has a compressed snapshot of the other parties embeddings
        and the server model. To calculate $\hat{y}$, Party $1$ uses
        its own embedding calculated at iteration $t$, and the
        embeddings and server model calculated at 
        time $t_0$, the latest communication iteration, 
        and compressed with $\mathcal{C}_m$.
    }
    \label{vflview.fig}
\end{figure}

\mbox{C-VFL} also generalizes previous work by supporting an arbitrary server model.
Previous work in VFL has either only analyzed
an arbitrary server model without local updates~\citep{chen2020vafl}, or 
analyzed local updates with a linear server model~\citep{liu2019communication,DBLP:journals/corr/abs-2012-12420,9415026}.
\mbox{C-VFL} is designed with an arbitrary server model,
allowing support for more complex prediction tasks
than those supported by previous VFL algorithms.

We summarize our main contributions in this work.
\begin{enumerate}[leftmargin=0cm,itemindent=.5cm,labelwidth=\itemindent,labelsep=0cm,align=left]
\item We introduce \mbox{C-VFL}
        with an arbitrary compression scheme. 
        Our algorithm generalizes previous work in VFL by 
        including both an arbitrary server model 
        and multiple local iterations.
    \item We prove convergence of \mbox{C-VFL} \rev{to a neighborhood of a fixed point}
            on non-convex objectives at a rate of 
            $O(\frac{1}{\sqrt{T}})$ for a fixed step size when the compression error is bounded over the course of training.
    We also prove that the algorithm convergence error goes 
        to zero for a diminishing step size if the compression 
        error diminishes as well. 
        Our work provides novel analysis for the effect of 
        compressing embeddings on convergence in a VFL algorithm.
        Our analysis also applies to Split Learning when uploads 
        to the server are compressed. 
\item       We provide convergence bounds on parameters in common compressors
    that can be used in \mbox{C-VFL}.
        In particular, we examine
        scalar quantization~\citep{DBLP:journals/bstj/Bennett48}, 
        lattice vector quantization~\citep{DBLP:journals/tit/ZamirF96}, 
        and top-$k$ sparsification~\citep{DBLP:conf/iclr/LinHM0D18}.    
\item We evaluate our algorithm by training on 
    MIMIC-III, CIFAR-10, and ModelNet10 datasets.
        We empirically show how \mbox{C-VFL} can reduce 
        the number of bits sent by over $90\%$ compared to 
        VFL with no compression
        without a significant loss in accuracy of the final model.  
\end{enumerate}

\paragraph{Related Work.}
\cite{DBLP:journals/mp/RichtarikT16,hardy2017private}
were the first works to propose Federated Learning algorithms for
vertically partitioned data.
\cite{chen2020vafl, pyvertical}  
propose the inclusion of an arbitrary server model 
in a VFL algorithm.
However, these works do not consider multiple local iterations,
and thus communicate at every iteration.  
\cite{liu2019communication},
\cite{DBLP:journals/corr/abs-2001-11154},
and \cite{9415026} 
all propose different VFL algorithms with local 
iterations for vertically partitioned data
but do not consider an arbitrary server model. 
\rev{Split Learning is a related concept to VFL~\cite{
DBLP:journals/jnca/GuptaR18}.
Split Learning can be thought of a special case 
of VFL when there is only one party. 
Recent works~\cite{He2020SplitFed, han2021accelerating} 
have extended Split Learning to a Federated Learning setting. 
However, these works focus on the HFL setting and
do not apply message compression.}
In contrast to previous works, our work addresses
a vertical scenario, 
an arbitrary server model,
local iterations, and message compression.

Message compression is a common topic in HFL scenarios, where participants exchange gradients determined by their local datasets.
Methods of gradient compression in HFL include
scalar quantization~\citep{DBLP:conf/icml/BernsteinWAA18}, 
vector quantization~\citep{DBLP:journals/tsp/ShlezingerCEPC21}, 
and top-$k$ sparsification~\citep{DBLP:conf/ijcai/ShiZWTC19}.
In \mbox{C-VFL}, compressed embeddings are shared, rather than compressed gradients.
Analysis in previous work on gradient compression in HFL 
does not apply to compression in VFL, as
the effect of embedding compression error on each 
party's partial derivatives may be complex.
No prior work has analyzed the impact of compression on convergence in VFL.%

\paragraph{Outline.} 
In Section~\ref{problem.sec}, we provide the problem formulation and
our assumptions. Section~\ref{alg.sec} presents the details of
\mbox{C-VFL}. In Section~\ref{main.sec}, we present
our main theoretical results. Our experimental results are given
in Section~\ref{exp.sec}. Finally, we conclude in Section~\ref{conclusion.sec}.

\section{Problem Formulation} \label{problem.sec}

We present our problem formulation and notation to be used in the rest of
the paper. 
We let $\lrVert{a}$ be the 2-norm of a vector $a$,
and let $\lrVert{\A}_{\mathcal{F}}$ be the Frobenius 
norm of a matrix $\A$.

We consider a set of $M$
parties $\{1, \ldots ,M\}$ and a server. 
The dataset $\X \in \mathbb{R}^{N \times D}$ is vertically partitioned 
a priori across the $M$ parties, where $N$ is the number of data samples
and $D$ is the number of features. 
The $i$-th row of $\X$ corresponds to a data sample $x^i$. 
For each sample $x^i$, 
a party $m$ holds a disjoint subset of the features, denoted $x_m^i$,
so that $x^i = [x_1^i, \ldots, x_M^i]$.
For each $x^i$, there is a corresponding label $y^i$.
Let $\y \in \mathbb{R}^{N \times 1}$ be the vector of 
all sample labels. 
We let $\X_m \in \mathbb{R}^{N \times D_m}$ be the local dataset of a party $m$,
where the $i$-th row correspond to data features $x_m^i$. 
We assume that the server and all parties have a copy 
of the labels $\y$. 
For scenarios
where the labels are private and only present at a single
party, the label holder can provide enough information for
the parties to compute gradients for some classes of 
model architectures~\citep{liu2019communication}.

Each party $m$ holds a set of model parameters
$\theta_m$ as well as a local \emph{embedding} function
$h_m(\cdot)$. 
The server holds a set of parameters $\theta_0$ called the
\emph{server model} and a loss
function $l(\cdot)$ that combines the \emph{embeddings}
$h_m(\theta_m ; x_m^i)$ from all parties.
Our objective is to minimize the following:
\begin{align}
    &F(\Theta ; \X ; \y)  \nonumber \\
    &~~~~~\coloneqq 
    \frac{1}{N} \sum_{i=1}^N l(\theta_0, h_1(\theta_1; x^i_1) , \ldots , h_M(\theta_M; x^i_M); y^i) 
    \label{problem.eq}
\end{align}
where $\Theta = [\theta_0^T, \ldots ,\theta_M^T]^T$
is the \emph{global model}. 
An example of a global model $\Theta$
is in Figure~\ref{vflmodel.fig}.

For simplicity, we let $m=0$ refer to the server,
and define $h_0(\theta_0 ; x^i) \coloneqq \theta_0$ 
for all $x^i$, where $h_0(\cdot)$ is equivalent to the identity function.
Let $h_m(\theta_m; x^i_m) \in \mathbb{R}^{P_m}$ for $m=0, \ldots, M$,
where $P_m$ is the size of the $m$-th embedding.
Let $\nabla_m F(\Theta ; \X ; \y) \coloneqq \frac{1}{N} \sum_{i=1}^N 
\nabla_{\theta_m} l(\theta_0, h_1(\theta_1; x^i_1), \ldots , h_M(\theta_M; x^i_M) ; y^i)$ 
be the partial derivatives for parameters $\theta_m$.

Let $\X^{\B}$ and $\y^{\B}$ be the set of samples and labels corresponding
to a randomly sampled mini-batch $\B$ of size $B$.
We let the stochastic partial derivatives for parameters $\theta_m$ be $\nabla_m F_{\B}(\Theta ; \X ; \y) \coloneqq \frac{1}{B} \sum_{x^i,y^i \in \X^{\B},\y^{\B}} 
\nabla_{\theta_m} l(\theta_0, h_1(\theta_1; x^i_1), \ldots , h_M(\theta_M; x^i_M) ; y)$. 
We may drop $\X$ and $\y$ from $F(\cdot)$ and $F_{\B}(\cdot)$.
With a minor abuse of notation, 
we let $h_m(\theta_m; \X_m^{\B})
\coloneqq \{h_m(\theta_m; x^{\B^1}_m), \ldots , h_m(\theta_m; x^{\B^B}_m) \}$ 
be the set of all party $m$ embeddings associated with mini-batch $\B$,
where $\B^i$ is the $i$-th sample in the mini-batch $\B$.
We let $\nabla_m F_{\B}(\Theta)$ and 
$\nabla_m F_{\B}(\theta_0, h_1(\theta_1; \X_1^{\B}), \ldots, h_M(\theta_M; \X_M^{\B}))$ 
be equivalent, and use them interchangeably. 

\begin{assumption}{1}
    \label{smooth.assum} \textbf{Smoothness}: 
    There exists positive constants ${L < \infty}$ and ${L_m < \infty}$, for ${m=0, \ldots, M}$,
    such that for all $\Theta_1$, $\Theta_2$,
    the objective function satisfies:
        \begin{align*}
            \lrVert{\nabla F(\Theta_1) - \nabla F(\Theta_2)} &\leq L \lrVert{\Theta_1 - \Theta_2} \\
            \lrVert{\nabla_m F_{\B}(\Theta_1) - \nabla_m F_{\B}(\Theta_2)} &\leq L_m \lrVert{\Theta_1 - \Theta_2}.
        \end{align*}
\end{assumption}
\begin{assumption}{2}
    \label{bias.assum} \textbf{Unbiased gradients}: 
    For $m=0, \ldots ,M$, for every mini-batch $\B$, 
        the stochastic partial derivatives are unbiased, i.e.,
            $\mathbb{E}_{\B}{\nabla_m F_{\B}(\Theta)} = \nabla_m F(\Theta)$.
\end{assumption}
\begin{assumption}{3}
    \label{var.assum} \textbf{Bounded variance}:
     For $m=0, \ldots ,M$, there exists constants ${\sigma_m < \infty}$
        such that the variance of the stochastic partial derivatives are bounded as:
            $\mathbb{E}_{\B}{\lrVert{\nabla_m F(\Theta) -  \nabla_m F_{\B}(\Theta)}^2} \leq \frac{\sigma_m^2}{B}$
        for a mini-batch $\B$ of size $B$.
\end{assumption}

Assumption~\ref{smooth.assum} bounds how fast the
gradient and stochastic partial derivatives can change. 
Assumptions~\ref{bias.assum} and
\ref{var.assum} require that the stochastic partial
derivatives are unbiased estimators of the true
partial derivatives with bounded variance.
Assumptions~\ref{smooth.assum}--\ref{var.assum}
are common assumptions in convergence analysis of 
gradient-based algorithms~\citep{
tsitsiklis1986distributed,
HogWild18,
bottou2018optimization}.
We note Assumptions~\ref{bias.assum}--\ref{var.assum} are similar to
the IID assumptions in HFL convergence analysis. However,
in VFL settings, all parties store identical sample IDs but different
subsets of features. Hence, there is no equivalent notion of a non-IID
distribution in VFL.

\begin{assumption}{4}
    \label{smooth2.assum} \textbf{Bounded Hessian}: 
    There exists positive constants $H_m$ for $m=0, \ldots, M$
    such that for all $\Theta$, the second partial derivatives of $F_{\B}$
    with respect to $h_m(\theta_m; \X_m^{\B})$ satisfy: 
        \begin{align}
           \|\nabla^2_{h_m(\theta_m; \X_m^{\B})} F_{\B}(\Theta)\|_{\mathcal{F}} \leq H_m
        \end{align}
    for any mini-batch $\B$.
\end{assumption}

\begin{assumption}{5}
    \label{bounded.assum} \textbf{Bounded Embedding Gradients}: 
    There exists positive constants $G_m$ for $m=0, \ldots, M$
    such that for all $\theta_m$, the stochastic embedding gradients are bounded:
        \begin{align}
            \|\nabla_{\theta_m} h_m (\theta_m; \X_m^{\B})\|_{\mathcal{F}} \leq G_m
        \end{align}
        for any mini-batch $\B$.
\end{assumption}

Since we are assuming
a Lipschitz-continuous loss function (Assumption~\ref{smooth.assum}),
we know the Hessian of $F$ is bounded.
Assumption~\ref{smooth2.assum} strengthens this assumption
slightly to also bound the Hessian over any mini-batch.
Assumption~\ref{bounded.assum} bounds the magnitude
of the partial derivatives with respect to embeddings.
This embedding gradient bound is necessary to ensure
convergence in the presence of embedding compression error
(see Appendix~\ref{lemmas.sec} for details).

\section{Algorithm} \label{alg.sec}

We now present \mbox{C-VFL}, 
a communication-efficient method for training a global model
with vertically partitioned data.
In each \emph{global round}, 
a mini-batch $\B$ is chosen randomly from all samples and
parties share necessary information for
local training on this mini-batch. 
Each party $m$, in parallel, runs block-coordinate
stochastic gradient descent on its local model parameters $\theta_m$ 
for $Q$ local iterations. 
\mbox{C-VFL} runs for a total of $R$ global rounds,
and thus runs for $T=RQ$ total local iterations.

For party $m$ to compute the stochastic gradient 
with respect to its features, %
it \rev{requires the embeddings computed by all other parties $j \neq m$.
In \mbox{C-VFL}, these embeddings are shared with the server then distributed to the parties.} 
We reduce communication cost by only sharing embeddings
every global round. Further, each party compresses their
embeddings before sharing.
We define a set of general compressors for compressing party embeddings
and the server model:
$\mathcal{C}_m(\cdot) : \mathbb{R}^{P_m} \rightarrow \mathbb{R}^{P_m}$ 
for $m=0, \ldots, M$.  
To calculate the gradient for data sample $x^i$,
party $m$ receives $\mathcal{C}_j(h_j(\theta_j ; x_j^i))$ 
from all parties $j \neq m$.
With this information,
a party $m$ can compute $\nabla_m F_{\B}$ and update its parameters $\theta_m$
for multiple local iterations.
Note that each party uses a stale view of the global model
to compute its gradient during these local iterations, as it is reusing
the embeddings it receives at the start of the round.
In Section~\ref{main.sec}, we show that \mbox{C-VFL} converges even though
parties use stale information.
An example view a party has of the global model during training 
is in Figure~\ref{vflview.fig}.
Here, $t$ is the current iteration and $t_0$ is the start of
the most recent global round, when embeddings were shared.

\begin{algorithm}[t]
    \begin{algorithmic}[1]
    \STATE {\textbf{Initialize:}} $\theta_m^0$ for all parties $m$ and server model $\theta_0^0$
    \FOR {$t \leftarrow 0, \ldots, T-1$}
        \IF{$t$ mod $Q = 0$}
            \STATE Randomly sample $\B^t \in \{\X, \y\}$
            \FOR {$m \leftarrow 1, \ldots, M$ in parallel}
                \STATE Send $\mathcal{C}_m(h_m(\theta_m^t ; \X^{\B^t}_m))$ to server %
            \ENDFOR
            \STATE $\hat{\Phi}^{t_0} \leftarrow 
                \{\mathcal{C}_0(\theta_0^t), \mathcal{C}_1(h_1(\theta_1^t)), 
                \ldots, \mathcal{C}_M(h_M(\theta_M^t))\}$
            \STATE Server sends $\hat{\Phi}^{t_0}$ to all parties 
        \ENDIF
        \FOR {$m \leftarrow 0, \ldots, M$ in parallel}
        \STATE $\hat{\Phi}_m^t \leftarrow \{ \hat{\Phi}_{-m}^{t_0} ; h_m(\theta_m^{t}; \X_m^{\B^{t_0}})\}$ 
            \STATE $\theta_m^{t+1} \leftarrow \theta_m^t - \eta^{t_0}
            \nabla_m F_{\B}(\hat{\Phi}_m^t ; \y^{\B^{t_0}})$  %
    	\ENDFOR
    \ENDFOR
    \end{algorithmic}
    \caption{Compressed Vertical Federated Learning}
    \label{cvfl.alg}
\end{algorithm}
Algorithm~\ref{cvfl.alg} details the procedure of
\mbox{C-VFL}. In each global round, 
when $t$ mod $Q = 0$, a mini-batch $\B$ is randomly sampled from $\X$
and the parties exchange the associated embeddings,
compressed using $\mathcal{C}_m(\cdot)$, via the server (lines 3-9).  
Each party $m$ completes $Q$ local iterations, 
using the compressed embeddings it received in iteration
$t_0$ and 
its own $m$-th uncompressed embeddings $h_m(\theta_m^{t}, \X_m^{\B^{t_0}})$.
We denote the set of embeddings that each party receives in iteration $t_0$ as: 
\begin{align}
    \hat{\Phi}^{t_0} \coloneqq 
\{ \mathcal{C}_0(\theta_0^{t_0}), 
    \mathcal{C}_1(h_1(\theta_1^{t_0})), \ldots , 
\mathcal{C}_M(h_M(\theta_M^{t_0})) \}.
\end{align}
We let $\hat{\Phi}_{-m}^{t_0}$ be the set of compressed embeddings from parties $j \neq m$,
and let $\hat{\Phi}_{m}^t \coloneqq \{ \hat{\Phi}_{-m}^{t_0} ; h_m(\theta_m^{t}; \X_m^{\B^{t_0}})\}$.
For each local iteration, each party $m$ updates $\theta_m$ 
by computing the stochastic partial derivatives 
$\nabla_m F_{\B}(\hat{\Phi}_m^t ; \y^{\B^{t_0}})$ and
applying a gradient step with step size $\eta^{t_0}$ (lines 11-14). 

\rev{A key difference of \mbox{C-VFL} from
previous VFL algorithms is the support of a 
server model with trainable parameters, allowing for arbitrary fusion networks. To
support such a model with multiple local iterations, the
server model parameters are shared with the parties.}
Also note that the same mini-batch is used for all $Q$ local iterations, thus
communication is only required every $Q$ iterations. 
Therefore, without any compression,
the total communication cost is $O(R\cdot M \cdot (B \cdot \sum_m P_m + |\theta_0|))$
for $R$ global rounds.
Our compression technique replaces $P_m$ and $|\theta_0|$ with
smaller values based on the compression factor.
For cases where embeddings, the batch size, and the server model are large, 
this reduction can greatly decrease the communication cost.

\paragraph*{Privacy.}
We now discuss privacy-preserving mechanisms for \mbox{C-VFL}. 
In HFL settings, model update or gradient information is shared in messages.
It has been shown that gradients 
can leak information about the raw data~\citep{DBLP:journals/tifs/PhongAHWM18, DBLP:conf/nips/GeipingBD020}.
However in {\mbox{C-VFL}}, parties only share embeddings and can
only calculate the partial derivatives associated with the
server model and their local models. Commonly proposed
HFL gradient attacks cannot be performed on \mbox{C-VFL}.
Embeddings may be vulnerable to model inversion attacks~\citep{MahendranV15},
which are methods by which an attacker can recover raw input to a model
using the embedding output and black-box access to the model. 
One can protect against such an
attack using homomorphic encryption~\citep{SecureBoost, hardy2017private} 
or secure multi-party computation~\citep{gu2021privacy}.
Alternatively, if the input to the server model is the sum of party
embeddings, then secure aggregation methods~\citep{DBLP:journals/corr/BonawitzIKMMPRS16} 
can be applied.
\rev{Several works have proposed privacy-perserving methods for VFL 
settings~\cite{SecureBoost, VFL_SMC_2015, Zheng2022Secret}
that are compatible with the {\mbox{C-VFL}} algorithm.}

Note that \mbox{C-VFL} assumes all parties have access to the labels.
For low-risk scenarios, such as predicting credit score,
labels may not need to be private among the parties.
In cases where labels are private, 
one can augment {\mbox{C-VFL}} to apply the method in \cite{liu2019communication} 
for gradient calculation without the need for sharing labels.
Our analysis in Section~\ref{main.sec} would still hold in this case, 
and the additional communication is reduced by the use of message compression.

\section{Analysis} \label{main.sec}

\begin{table*}[t]
\small
\centering
\caption{
Choice of common compressor parameters to achieve 
a convergence rate of $O(1/\sqrt{T})$.
$P_m$ is the size of the $m$-th embedding.
In scalar quantization, we let there be $2^q$ quantization levels,
and let $h_{\max}$ and $h_{\min}$ be respectively the
maximum and minimum components in $h_m(\theta_m^t; x_m^i)$ 
for all iterations $t$, parties $m$, and $x_m^i$.
We let $V$ be the size of the lattice cell in vector quantization.
We let $k$ be the number of parameters sent in an embedding after top-$k$ sparsification, and $(\|h\|^2)_{\max}$ be the maximum value
of $\|h_m(\theta_m^t; x_m^i)\|^2$ 
for all iterations $t$, parties $m$, and $x_m^i$. 
}
\label{comp.table}
\vskip 0.1in
\resizebox{\textwidth}{!}{
\begin{tabular}{lccc}
    \toprule 
    & \textbf{Scalar Quantization}
    & \textbf{Vector Quantization}
    & \textbf{Top-$k$ Sparsification}\\
    \midrule
    \textbf{Parameter Choice}
    & $q = \Omega \left(\log_2 \left(B P_m(h_{\max} - h_{\min})^2\sqrt{T}\right) \right)$ 
    & $V = O\left( \frac{1}{BP_m\sqrt{T}} \right)$ 
    & $k = \Omega \left(P_m - \frac{P_m}{B(\lrVert{h}^2)_{\max}\sqrt{T}}\right)$ \\ 
    \textbf{Compression Error}
    & $\mathcal{E}_m^{t_0} \leq BP_m\frac{(h_{\max} - h_{\min})^2}{12}2^{-2q}$
    & $\mathcal{E}_m^{t_0} \leq \frac{VBP_m}{24}$
    & $\mathcal{E}_m^{t_0} \leq B(1-\frac{k}{P_m})(\|h\|^2)_{\max}$\\
    \bottomrule
\end{tabular}}
\end{table*}

In this section, we discuss our analytical approach 
and present our theoretical results.
We first define the compression error associated with $\mathcal{C}_m(\cdot)$:
\begin{definition}
    \label{compress.assum} \textbf{Compression Error}: 
    Let vectors $\epsilon_m^{x^i}$ for $m=0, \ldots ,M$,
    be the compression errors of $\mathcal{C}_m(\cdot)$ on a data sample $x^i$:
        $\epsilon_m^{x^i} \coloneqq  \mathcal{C}_m(h_m (\theta_m; x^i)) - h_m (\theta_m; x^i)$. 
    Let $\epsilon_m^{t_0}$ be the $P_m \times B$ matrix 
    with $\epsilon_m^{x^i}$ for all data
samples $x^i$ in mini-batch $\B^{t_0}$ as the columns.
We denote the expected squared message compression error from party $m$ at round $t_0$ 
as $\mathcal{E}_m^{t_0} \coloneqq  \mathbb{E}\lrVert{\epsilon_m^{t_0}}_{\mathcal{F}}^2$.
\end{definition}
Let $\hat{\G}^t$ be the stacked partial derivatives at iteration $t$: 
$$\hat{\G}^t \coloneqq [(\nabla_0 F_{\B}(\hat{\Phi}_0^t ; \y^{\B^{t_0}}))^T, \ldots, (\nabla_M F_{\B}(\hat{\Phi}_M^t ; \y^{\B^{t_0}}))^T]^T.$$
The model $\Theta$ evolves as:
\begin{align}
    \Theta^{t+1} = \Theta^t - \eta^{t_0} \hat{\G}^t.
\end{align}
We note the reuse of the mini-batch of $\B^{t_0}$ for $Q$ iterations in this recursion.
This indicates that the stochastic gradients are not unbiased during
local iterations $t_0+1 \leq t \leq t_0+Q-1$. 
Using conditional expectation, we can apply Assumption~\ref{bias.assum} 
to the gradient calculated at iteration $t_0$ when there is no compression error.
We define $\Phi^{t_0}$ to be the set of embeddings that would be received
by each party at iteration $t_0$ if no compression error were applied:
\begin{align}
    \Phi^{t_0} &\coloneqq 
    \{ \theta_0^{t_0}, 
    h_1(\theta_1^{t_0}), \ldots , 
    h_M(\theta_M^{t_0}) \}.
\end{align}
We let $\Phi_{-m}^{t_0}$ be the set of embeddings from parties $j \neq m$,
and let $\Phi_{m}^t \coloneqq \{ \Phi_{-m}^{t_0} ; h_m(\theta_m^{t}; \X_m^{\B^{t_0}})\}$.
Then, if we take expectation over $\B^{t_0}$ conditioned
on previous global models $\Theta^t$ up to $t_0$:
\begin{align}
    &\mathbb{E}_{\B^{t_0}}[\nabla_m F_{\B} (\Phi_m^{t_0}) ~|~ \{\Theta^{\tau}\}_{\tau=0}^{t_0}]     = \nabla_m F( \Phi_m^{t_0} ).
    \label{bias_imp2.assum}
\end{align}
With the help of (\ref{bias_imp2.assum}), we can prove convergence by
bounding the difference between the gradient at 
the start of each global round and those calculated during local iterations
(see the proof of Lemma~\ref{diff.lemma} in Appendix~\ref{lemmas.sec} for details).

To account for compression error, 
using the chain rule and 
Taylor series expansion,  
we obtain:
\begin{lemma}{1} \label{error.lemma} 
    Under Assumptions~\ref{smooth2.assum}-\ref{bounded.assum}, the norm of the difference
    between the objective function value with compressed and uncompressed embeddings
    is bounded as:
\begin{align}
    \mathbb{E} \|\nabla_m F_{\B}(\hat{\Phi}_m^t) - \nabla_m F_{\B}(\Phi_m^t)\|^2
    &\leq H_m^2 G_m^2 \sum_{j=0, j \neq m}^M 
    \mathcal{E}_j^{t_0}.
    \nonumber
\end{align}
\end{lemma}
The proof of Lemma~\ref{error.lemma} is given in Appendix~\ref{lemmas.sec}.
Using Lemma~\ref{error.lemma}, we can bound the effect of compression error
on convergence.

We present our main theoretical results.
All proofs are provided in Appendix~\ref{proofs.sec}.

\begin{theorem} \label{main.thm} \textbf{Convergence with fixed step size}: 
    Under Assumptions~\ref{smooth.assum}-\ref{bounded.assum}, 
    if $\eta^{t_0} = \eta$ for all iterations and satisfies
$\eta^{t_0} \leq \frac{1}{16Q\max\{L,\max_m L_m\}}$,
then the average squared gradient over
$R$ global rounds of Algorithm 1 is bounded:
\begin{align}
    &\frac{1}{R} \sum_{t_0=0}^{R-1} \Etot{\lrVert{\nabla F(\Theta^{t_0})}^2} 
    \nonumber \\
    &~~~~~~~\leq \frac{4\left[F(\Theta^0) - \Etot{F(\Theta^T)}\right]}{\eta T}
    + 6\eta QL\sum_{m=0}^M \frac{\sigma_m^2}{B} 
    \nonumber \\ &~~~~~~~~~~~~~~~~~~~
    + \frac{92Q^2}{R}  \sum_{m=0}^M H_m^2 G_m^2 
    \sum_{t_0=0}^{R-1} \sum_{j=0, j \neq m}^M \mathcal{E}_j^{t_0}.
    \label{main.eq}
\end{align}
\end{theorem}
The first term in (\ref{main.eq}) is based on the
difference between the initial model and final model of the algorithm. 
The second term is the error associated with the
variance of the stochastic gradients 
and the Lipschitz constants $L$ and $L_m$'s.  
The third term relates to the average compression error over all iterations.
The larger the error introduced by a compressor,
the larger the convergence error is. 
We note that setting $\mathcal{E}_j^{t_0}=0$ 
for all parties and iterations
provides an error bound on VFL without compression
and is an improvement over the bound in \cite{liu2019communication}
in terms of $Q$, $M$, and $B$. 
The second and third terms include a coefficient relating to
local iterations.  
As the number of local iterations $Q$ increases, the convergence error 
increases. However, increasing $Q$ also has the effect of reducing the 
number of global rounds. Thus, it may be beneficial to have
$Q>1$ in practice. 
We explore this more in experiments in Section~\ref{exp.sec}.
The second and third terms scale with $M$, the number of parties. 
However, VFL scenarios typically have a small 
number of parties~\citep{kairouz2019advances}, and thus $M$ plays
a small role in convergence error.
We note that when $M=1$ and $Q=1$, Theorem~\ref{main.thm} applies to 
Split Learning~\citep{DBLP:journals/jnca/GuptaR18}
when only uploads to the server are compressed. 

\begin{remark}{\normalfont
Let $\mathcal{E} = \frac{1}{R}\sum_{t_0=0}^{R-1} \sum_{m=0}^M 
\mathcal{E}_m^{t_0}$.
If $\eta^{t_0} = \frac{1}{\sqrt{T}}$ for all global rounds $t_0$, for $Q$ and $B$ independent of $T$,
then 
$$\frac{1}{R}{\sum_{t_0=0}^{R-1}} \Etot{\|\nabla F(\Theta^{t_0})\|^2} = O\left( \frac{1}{\sqrt{T}} + \mathcal{E}\right).$$
This indicates that if
$\mathcal{E} = O(\frac{1}{\sqrt{T}})$ 
then we can achieve a convergence rate of 
$O(\frac{1}{\sqrt{T}})$.
Informally, this means that \mbox{C-VFL} can afford compression error and not
worsen asymptotic convergence when this condition is satisfied.
We discuss how this affects commonly used compressors in practice later in the section.

}\end{remark}

We consider a diminishing step size in the following:
\begin{theorem} \label{main2.thm} \textbf{Convergence with diminishing step size}:
    Under Assumptions~\ref{smooth.assum}-\ref{bounded.assum}, 
    if $0 < \eta^{t_0} < 1$ satisfies
    $\eta^{t_0} \leq \frac{1}{16Q\max\{L,\max_m L_m\}}$,
then the minimum squared gradient over
$R$ global rounds of Algorithm 1 is bounded:
\begin{align*}
    &\min_{t_0=0,\ldots,R-1} \Etot{\lrVert{\nabla F(\Theta^{t_0})}^2} =
    \nonumber \\
    &O\left(\frac{1}{\sum_{t_0=0}^{R-1} \eta^{t_0}}
    + \frac{\sum_{t_0=0}^{R-1} (\eta^{t_0})^2}{\sum_{t=0}^{T-1} \eta^{t_0}} 
+ \frac{\sum_{t_0=0}^{R-1} \sum_{m=0}^M \eta^{t_0} \mathcal{E}_m^{t_0}}
{\sum_{t_0=0}^{R-1} \eta^{t_0}} \right).
\end{align*}
If $\eta^{t_0}$ and $\mathcal{E}_m^{t_0}$ satisfy
    $\sum_{t_0=0}^{\infty} \eta^{t_0} = \infty$,  
    $\sum_{t_0=0}^{\infty} (\eta^{t_0})^2 < \infty$, 
    and $\sum_{t_0=0}^{\infty} \sum_{m=0}^M \eta^{t_0} \mathcal{E}_m^{t_0} < \infty$,
    then 
    $\min_{t_0=0,\ldots,R-1} \Etot{\lrVert{\nabla F(\Theta^{t_0})}^2}
    \rightarrow 0$ as $R \rightarrow \infty$.
\end{theorem}
According to Theorem~\ref{main2.thm},
the product of the step size and the 
compression error must be summable over all iterations.
In the next subsection, we discuss how to choose common
compressor parameters to ensure this property is satisified.
We also see in Section~\ref{exp.sec} that good results can
be achieved empirically without diminishing the step size or 
compression error.%

\paragraph*{Common Compressors.}
In this section, we show how to choose common compressor parameters 
to achieve a convergence rate of $O(\frac{1}{\sqrt{T}})$ 
in the context of Theorem~\ref{main.thm},
and guarantee convergence in the context of Theorem~\ref{main2.thm}.
We analyze three common compressors: %
a uniform scalar quantizer~\citep{DBLP:journals/bstj/Bennett48}, 
a $2$-dimensional hexagonal lattice quantizer~\citep{DBLP:journals/tit/ZamirF96}, 
and top-$k$ sparsification~\citep{DBLP:conf/iclr/LinHM0D18}.
For uniform scalar quantizer,
we let there be $2^q$ quantization levels.
For the lattice vector quantizer,
we let $V$ be the volume of each lattice cell.
For top-$k$ sparsification, we let $k$ be the number of embedding components sent
in a message.
In Table~\ref{comp.table}, 
we present the choice of compressor parameters in order to 
achieve a convergence rate of $O(\frac{1}{\sqrt{T}})$ 
in the context of Theorem~\ref{main.thm}.
We show how we calculate these bounds in Appendix~\ref{common.sec} and 
provide some implementation details for their use. %
We can also use Table~\ref{comp.table} to choose compressor parameters
    to ensure convergence in the context of Theorem~\ref{main2.thm}. 
    Let $\eta^{t_0}=O(\frac{1}{t_0})$, where $t_0$ is the current round.
    Then setting $T=t_0$ in Table~\ref{comp.table} 
    provides a choice of compression parameters at each iteration 
    to ensure the compression error diminishes at a rate of $O(\frac{1}{\sqrt{t_0}})$,
    guaranteeing convergence.
Diminishing compression error 
can be achieved by increasing the number of quantization levels,
decreasing the volume of lattice cells, or increasing the number of
components sent in a message.

\section{Experiments} \label{exp.sec}
We present experiments to examine the performance of \mbox{C-VFL} 
in practice.
The goal of our experiments is to examine the effects that different
compression techniques have on training, 
and investigate the accuracy/communication trade-off empirically.
We run experiments on three datasets: 
the MIMIC-III dataset~\citep{johnson2016mimic}, 
the \mbox{ModelNet10} dataset~\citep{wu20153d}, and
the CIFAR-10 dataset~\citep{krizhevsky2009learning}.
We provide more details on the datasets and training procedure in Appendix~\ref{exp_details.sec},
as well as additional plots and experiments in Appendix~\ref{additional.sec}.

\textbf{MIMIC-III: }
MIMIC-III is an anonymized hospital patient time series dataset.
In MIMIC-III, the task is binary classification to predict in-hospital
mortality. We train with a set of $4$ parties, each storing $19$ of the $76$ features.
Each party trains an LSTM and the server trains two fully-connected layers.
We use a fixed step size of $0.01$, a batch size of $1000$, and 
train for $1000$ epochs.

\textbf{CIFAR-10: }
CIFAR-10 is an image dataset for object classification. 
We train with a set of $4$ parties, each storing a different quadrant of every image. 
Each party trains ResNet18, and the server trains a fully-connected layer.
We use a fixed step size of $0.0001$ and a batch size of $100$, and
train for $200$ epochs.

\textbf{ModelNet10: }
ModelNet10 is a set of CAD models, each with images of $12$ different camera views.
The task of ModelNet10 is classification of images into $10$ object classes.
We run experiments with both a set of $4$ and $12$ parties, 
where parties receive $3$ or $1$ view(s) of each CAD model, respectively.
Each party's network consists of two convolutional layers and a fully-connected layer, and
the server model consists of a fully-connected layer.
We use a fixed step size of $0.001$ and a batch size of $64$, and
train for $100$ epochs.

We consider the three compressors discussed in Section~\ref{main.sec}:
a uniform scalar quantizer, a $2$-dimensional 
hexagonal lattice (vector) quantizer, and top-$k$ sparsification.
For both quantizers, 
the embedding values need to be bounded.
In the case of the models used for MIMIC-III and CIFAR-10, 
the embedding values are already bounded,
but the CNN used for ModelNet10 may have unbounded embedding values.
We scale embedding values for \mbox{ModelNet10} to the range $[0,1]$.
We apply subtractive dithering
to both the scalar quantizer~\citep{wannamaker1997theory} 
and vector quantizer~\citep{DBLP:journals/tsp/ShlezingerCEPC21}.

In our experiments, each embedding component is a $32$-bit float.
Let $b$ be the bits per component we compress to.
For the scalar quantizer, this means there are $2^b$ quantization levels.
For the $2$-D vector quantizer, this means there are $2^{2b}$ 
vectors in the codebook. The volume $V$ of the vector quantizer is a
function of the number of codebook vectors.
For top-$k$ sparsification, $k = P_m \cdot \frac{b}{32}$
as we use $32$-bit components. 
We train using \mbox{C-VFL} 
and consider cases where $b=2$, $3$, and $4$.
We compare with a case where $b=32$. 
\rev{This corresponds to a standard VFL algorithm 
without embedding compression~\citep{liu2019communication}}, 
acting as a baseline for accuracy.

\begin{figure}[t]
    \begin{subfigure}{0.23\textwidth}
        \centering
        \includegraphics[width=\textwidth]{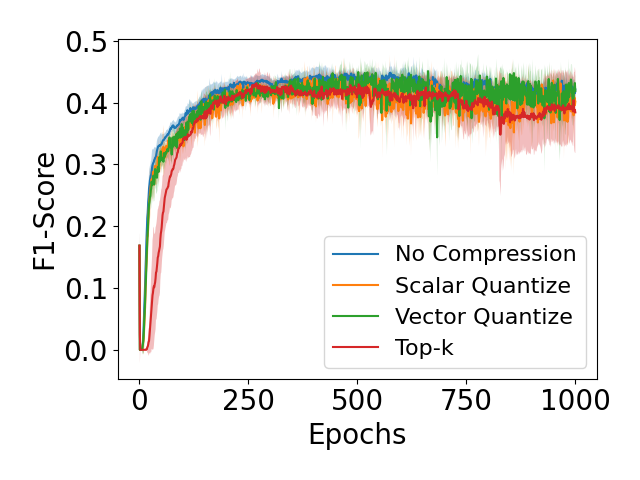}
        \caption{MIMIC-III by epochs}
        \label{mimic.fig}
    \end{subfigure}
    \begin{subfigure}{0.23\textwidth}
        \centering
        \includegraphics[width=\textwidth]{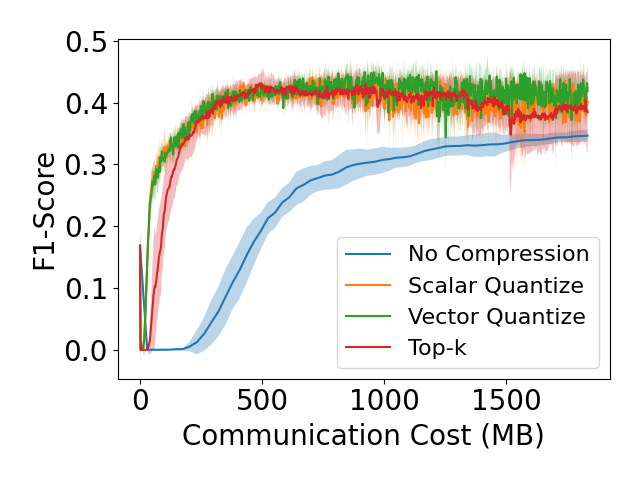}
        \caption{MIMIC-III by cost}
        \label{mimiccomm.fig}
    \end{subfigure}
    \begin{subfigure}{0.23\textwidth}
        \centering
        \includegraphics[width=\textwidth]{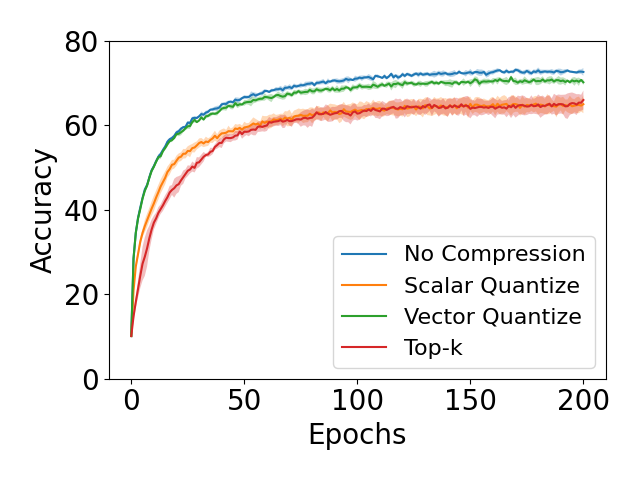}
        \caption{CIFAR-10 by epochs}
        \label{cifarepoch.fig}
    \end{subfigure}
    \begin{subfigure}{0.23\textwidth}
        \centering
        \includegraphics[width=\textwidth]{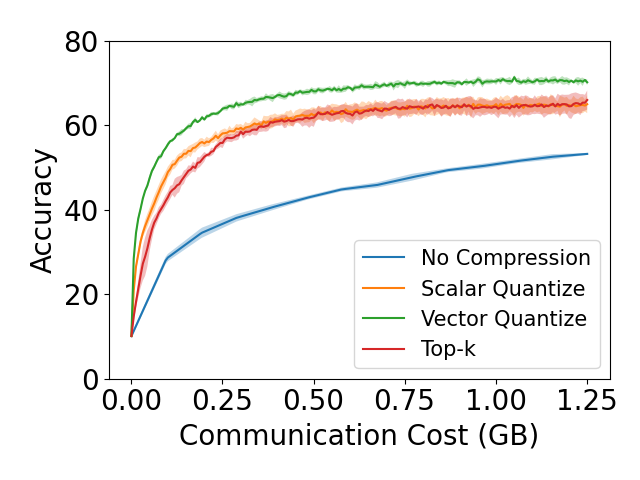}
        \caption{CIFAR-10 by cost}
        \label{cifarcomm.fig}
    \end{subfigure}
    \begin{subfigure}{0.23\textwidth}
        \centering
        \includegraphics[width=\textwidth]{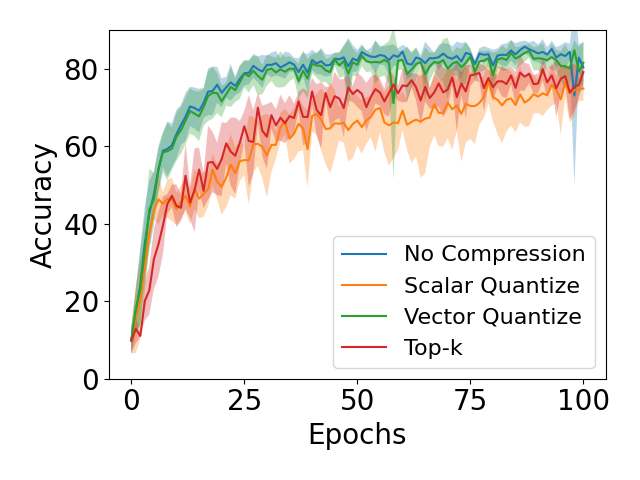}
        \caption{ModelNet10 by epochs}
        \label{mvcnn.fig}
    \end{subfigure}
    \hfill
    \begin{subfigure}{0.23\textwidth}
        \centering
        \includegraphics[width=\textwidth]{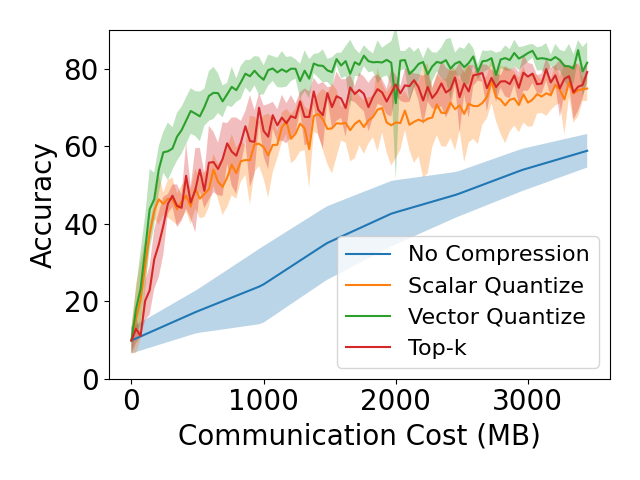}
        \caption{ModelNet10 by cost}
        \label{mvcnncomm.fig}
    \end{subfigure}
    \caption{\mbox{C-VFL} when compressing to $2$ bits per component.
        The solid lines are the mean of $5$ runs, while the shaded region represents
        the standard deviation.
        We show test $F_1$-Score on MIMIC-III dataset and test accuracy
    on CIFAR-10 and ModelNet10 dataset, plotted by epochs and communication cost.}
    \label{acc.fig}
\end{figure}

In Figure \ref{acc.fig}, 
we plot the test $F_1$-Score and test accuracy
for MIMIC-III, CIFAR-10, and ModelNet10 when training with $b=2$. 
We use $F_1$-Score for MIMIC-III as the in-hospital mortality
prediction task has high class imbalance; most people in the dataset 
did not die in the hospital.  
For these experiments, we let $M=4$ for ModelNet10.
The solid line in each plot represents
the average accuracy over five runs, while the shaded regions
represent one standard deviation. 
In Figures \ref{mimic.fig}, \ref{cifarepoch.fig}, and~\ref{mvcnn.fig}, 
we plot by the number of training epochs. 
We can see in all cases, although convergence can
be a bit slower, training with compressed embeddings 
still reaches similar accuracy to no compression.
In Figures \ref{mimiccomm.fig}, \ref{cifarcomm.fig}, and~\ref{mvcnncomm.fig}, 
we plot by the communication cost in bytes. The cost of communication
includes both the upload of (compressed) embeddings to the server and download of
embeddings and server model to all parties.
We can see that by compressing embeddings, we can reach higher
accuracy with significantly less communication cost.
In all datasets, the compressors reach similar accuracy to each other, 
though top-$k$ sparsification performs slightly worse than the others 
 on MIMIC-III,
while vector quantization performs the best in both on CIFAR-10 and ModelNet10.

\begin{table}[t]
\caption{MIMIC-III maximum $F_1$-Score reached during training,
and communication cost to reach a target test $F_1$-Score of $0.4$.
Value shown is the mean of $5$ runs, $\pm$ the standard deviation.
In these experiments, $Q=10$ and $M=4$.}
\label{acc_mimic.table}
\vskip 0.1in
\small
\centering
\resizebox{0.4\textwidth}{!}{
\begin{tabular}{lcccc}
    \toprule 
    \textbf{Compressor} & \subhead{Max $F_1$-Score} &\subhead{Cost (MB)}\\
                        & \subhead{Reached}         &\subhead{Target $=0.4$}\\
    \midrule
    \midrule
    None $b=32$                 &0.448 $\pm$ 0.010 & 3830.0 $\pm$ 558.2 \\
    \midrule
    Scalar $b=2$                &0.441 $\pm$ 0.018 & 233.1  $\pm$ 28.7  \\
    Vector $b=2$                &0.451 $\pm$ 0.021 & 236.1  $\pm$ 17.9  \\
    Top-$k$ $b=2$               &0.431 $\pm$ 0.016 & 309.8  $\pm$ 93.6  \\
    \midrule
    Scalar $b=3$                &0.446 $\pm$ 0.011 & 343.1  $\pm$ 18.8  \\
    Vector $b=3$                &0.455 $\pm$ 0.020 & 330.5  $\pm$ 10.6  \\
    Top-$k$ $b=3$               &0.435 $\pm$ 0.030 & 470.7  $\pm$ 116.8  \\
    \midrule
    Scalar $b=4$                &0.451 $\pm$ 0.020 & 456.0  $\pm$ 87.8   \\
    Vector $b=4$                &0.446 $\pm$ 0.017 & 446.5  $\pm$ 21.3   \\
    Top-$k$ $b=4$               &0.453 $\pm$ 0.014 & 519.1  $\pm$ 150.4  \\
    \bottomrule                                                                               
\end{tabular}
}
\end{table}

\begin{table}[t]
\caption{CIFAR-10 maximum test accuracy reached during training,
and communication cost to reach a target accuracy of $70\%$.
Value shown is the mean of $5$ runs, $\pm$ the standard deviation.
A ``--" indicates that the target was not reached during training.
We let $Q=10$ and $M=4$.}
\label{acc_cifar.table}
\vskip 0.1in
\small
\centering
\resizebox{0.4\textwidth}{!}{
\begin{tabular}{lcc}
    \toprule 
    \textbf{Compressor} & \subhead{Max Accuracy}    &\subhead{Cost (GB)}\\
    & \subhead{Reached}         &\subhead{Target $=70\%$}\\
    \midrule
    \midrule
    None $b=32$                 &73.18\% $\pm$ 0.44\% & 7.69 $\pm$ 0.35 \\
    \midrule
    Scalar $b=2$                &65.16\% $\pm$ 1.85\% & --              \\
    Vector $b=2$                &71.43\% $\pm$ 0.47\% & 0.68 $\pm$ 0.06 \\
    Top-$k$ $b=2$               &66.02\% $\pm$ 2.24\% & --              \\
    \midrule
    Scalar $b=3$                &71.49\% $\pm$ 1.05\% & 1.22 $\pm$ 0.17 \\
    Vector $b=3$                &72.50\% $\pm$ 0.40\% & 0.81 $\pm$ 0.05 \\
    Top-$k$ $b=3$               &71.56\% $\pm$ 0.81\% & 1.24 $\pm$ 0.22  \\
    \midrule
    Scalar $b=4$                &71.80\% $\pm$ 1.18\% & 1.72 $\pm$ 0.26  \\
    Vector $b=4$                &73.17\% $\pm$ 0.39\% & 0.98 $\pm$ 0.08  \\
    Top-$k$ $b=4$               &72.03\% $\pm$ 1.77\% & 1.43 $\pm$ 0.26  \\
    \bottomrule                                                                               
\end{tabular}
}
\end{table}

\begin{table}[t]
    \caption{ModelNet10 maximum test accuracy reached during training,
and communication cost to reach a target accuracy of $75\%$.
Value shown is the mean of $5$ runs, $\pm$ the standard deviation.
    We let $Q=10$ and $M=4$.}
\label{acc_modelnet.table}
\vskip 0.1in
\small
\centering
\resizebox{0.45\textwidth}{!}{
    \begin{tabular}{lcc}
    \toprule 
    \textbf{Compressor} &  \subhead{Max Accuracy}&      \subhead{Cost (MB)}\\
    &  \subhead{Reached}&          \subhead{Target $=75\%$}\\
    \midrule
    \midrule
    None $b=32$                 &85.68\% $\pm$ 1.57\% & 9604.80 $\pm$ 2933.40\\
    \midrule
    Scalar $b=2$                &76.94\% $\pm$ 5.87\% & 1932.00 $\pm$ 674.30 \\
    Vector $b=2$                &84.80\% $\pm$ 2.58\% & 593.40 $\pm$ 170.98  \\
    Top-$k$ $b=2$               &79.91\% $\pm$ 2.86\% & 1317.90 $\pm$ 222.95 \\
    \midrule
    Scalar $b=3$                &81.32\% $\pm$ 1.61\% & 1738.80 $\pm$ 254.79 \\
    Vector $b=3$                &85.66\% $\pm$ 1.36\% & 900.45 $\pm$ 275.01  \\
    Top-$k$ $b=3$               &81.63\% $\pm$ 1.24\% & 1593.90 $\pm$ 225.34  \\
    \midrule
    Scalar $b=4$                &81.19\% $\pm$ 1.88\% & 2194.20 $\pm$ 266.88  \\
    Vector $b=4$                &85.77\% $\pm$ 1.69\% & 1200.60 $\pm$ 366.68  \\
    Top-$k$ $b=4$               &83.50\% $\pm$ 1.21\% & 1821.60 $\pm$ 241.40  \\
    \bottomrule                                                                               
\end{tabular}
}
\end{table}

\begin{table}[t]
    \caption{MIMIC-III time in seconds to reach a target $F_1$-Score for
    different local iterations $Q$ and communication latency $t_c$ 
    with vector quantization and $b=3$.
    Value shown is the mean of $5$ runs, $\pm$ one standard deviation.
    }
    \label{local.table}
\vskip 0.1in
    \centering
    \resizebox{0.45\textwidth}{!}{
    \begin{tabular}{lccc}
        \toprule 
        $t_c$ & \multicolumn{3}{c}{\textbf{Time to Reach Target $F_1$-Score $0.45$}}\\
                                  & \subhead{$Q=1$}  & \subhead{$Q=10$} & \subhead{$Q=25$}\\
        \midrule
        \midrule
        $1$      &   694.53 $\pm$  150.75 &  470.86 $\pm$ 235.35 & 445.21 $\pm$  51.44  \\
        $10$     &  1262.78 $\pm$  274.10 &  512.82 $\pm$ 256.32 & 461.17 $\pm$  53.29 \\
        $50$     &  3788.32 $\pm$  822.30 &  699.30 $\pm$ 349.53 & 532.12 $\pm$  61.49 \\
        $200$    & 13259.14 $\pm$ 2878.04 & 1398.60 $\pm$ 699.05 & 798.19 $\pm$  92.23 \\
        \bottomrule
    \end{tabular}
    }
\end{table}

In Tables~\ref{acc_mimic.table}, \ref{acc_cifar.table} and~\ref{acc_modelnet.table}, we show the maximum 
test accuracy reached during training and the communication
cost to reach a target accuracy for MIMIC-III, CIFAR-10, and ModelNet10. 
We show results for all three
compressors with $b=2$, $3$, and $4$ bits per component, as 
well as the baseline of $b=32$. 
For the MIMIC-III dataset, we show the maximum test $F_1$-Score 
reached and the total communication cost of reaching an $F_1$-Score of $0.4$.
The maximum $F_1$-Score for each case is within a standard deviation of each other.
However, the cost to reach target score is much smaller 
as the value of $b$ decreases for all compressors.
We can see that when $b=2$, we can achieve over $90\%$
communication cost reduction over no compression to reach a target $F_1$-Score.

\rev{For the CIFAR-10 and ModelNet10 datasets, Tables~\ref{acc_cifar.table} and~\ref{acc_modelnet.table}
show the maximum test accuracy reached and the total communication cost of reaching a target accuracy.}
We can see that, \rev{for both datasets,} vector quantization tends to outperform both scalar
quantization and top-$k$ quantization.
Vector quantization benefits from considering components jointly,
and thus can have better reconstruction quality 
than scalar quantization and top-$k$ sparsification~\citep{woods2006multidimensional}.
\begin{figure}[t]
    \begin{subfigure}{0.23\textwidth}
        \centering
        \includegraphics[width=\textwidth]{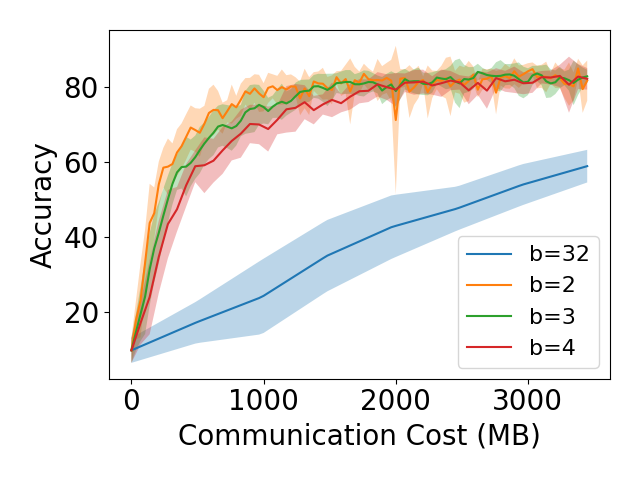}
        \caption{Parties $M=4$}
        \label{mvcnncomm4.fig}
    \end{subfigure}
    \hfill
    \begin{subfigure}{0.23\textwidth}
        \centering
        \includegraphics[width=\textwidth]{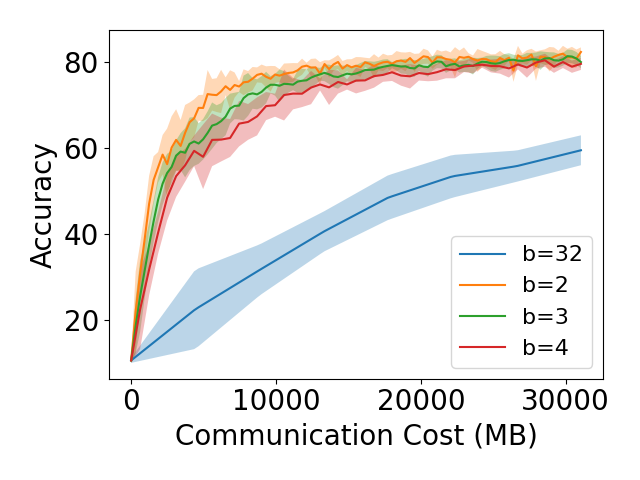}
        \caption{Parties $M=12$}
        \label{mvcnncomm12.fig}
    \end{subfigure}
    \caption{Communication cost of training on ModelNet10 with vector quantization.
        The solid lines are the mean of $5$ runs, and the shaded region represents
        one standard deviation.
    }
    \label{clients.fig}
\end{figure}

In Table~\ref{local.table}, we consider the communication/computation tradeoff
of local iterations.
We show how the number of local iterations
affects the time to reach a target $F_1$-Score in the MIMIC-III dataset. 
We train \mbox{C-VFL} with vector quantization $b=3$ and
set the local iterations $Q$ to $1$, $10$, and $25$. 
\rev{Note that the $Q=1$ case corresponds to adding embedding compression 
to previously proposed VFL algorithms that do not have multiple local 
iterations~\citep{FDML, pyvertical}.}
We simulate a scenario where computation time for training a mini-batch
of data at each party takes $10$ ms, and communication
of embeddings takes a total of $1$, $10$, $50$, and $200$ ms roundtrip. 
These different communication latencies correspond to the 
distance between the parties and the server:
within the same cluster, on the same local network, 
within the same region, and across the globe.
According to Theorem~\ref{main.thm},
increasing the number of local iterations $Q$ %
increases convergence error. However, the target test accuracy 
is reached within less time when $Q$ increases.
The improvement over $Q=1$ local iterations increases 
as the communication latency increases. In systems where communication
latency is high, it may be beneficial to increase the number of local iterations. 
The choice of $Q$ will depend on the accuracy requirements 
of the given prediction task and the time constraints on the prediction problem.

Finally, in Figure~\ref{clients.fig},
we plot the test accuracy of ModelNet10 against the communication
cost when using vector quantization with $b=2$, $3$, $4$, and $32$.
We include plots for $4$ and $12$ parties.
We note that changing the number of parties changes the 
global model structure $\Theta$ as well.
We can see in both cases that smaller values of $b$ 
reach higher test accuracies at lower communication cost.
The total communication cost is larger with $12$ parties,
but the impact of increasing compression is similar for
both $M=4$ and $M=12$.

\section{Conclusion} \label{conclusion.sec}
We proposed \mbox{C-VFL}, a distributed communication-efficient
algorithm for training a model over vertically partitioned data.
We proved convergence of the algorithm at a 
rate of $O( \frac{1}{\sqrt{T}} )$, and we showed experimentally
that communication cost could be reduced by over $90\%$ without a significant
decrease in accuracy.
For future work, we seek to relax our bounded gradient assumption
and explore the effect of adaptive compressors.

\rev{
\section*{Acknowledgements}
This work was supported by the Rensselaer-IBM AI Research Collaboration (\href{http://airc.rpi.edu}{http://airc.rpi.edu}), part of the IBM AI Horizons Network (\href{http://ibm.biz/AIHorizons}{http://ibm.biz/AIHorizons}),
and the National Science Foundation under grant CNS-1553340.
}

\balance
\bibliography{references}

\begin{thebibliography}{55}
\providecommand{\natexlab}[1]{#1}
\providecommand{\url}[1]{\texttt{#1}}
\expandafter\ifx\csname urlstyle\endcsname\relax
  \providecommand{\doi}[1]{doi: #1}\else
  \providecommand{\doi}{doi: \begingroup \urlstyle{rm}\Url}\fi

\bibitem[Bennett(1948)]{DBLP:journals/bstj/Bennett48}
Bennett, W.~R.
\newblock Spectra of quantized signals.
\newblock \emph{Bell Syst. Tech. J.}, 27\penalty0 (3):\penalty0 446--472, 1948.

\bibitem[Bernstein et~al.(2018)Bernstein, Wang, Azizzadenesheli, and
  Anandkumar]{DBLP:conf/icml/BernsteinWAA18}
Bernstein, J., Wang, Y., Azizzadenesheli, K., and Anandkumar, A.
\newblock {SIGNSGD:} compressed optimisation for non-convex problems.
\newblock \emph{Proc. Int. Conf. on Machine Learn.}, 2018.

\bibitem[Bonawitz et~al.(2016)Bonawitz, Ivanov, Kreuter, Marcedone, McMahan,
  Patel, Ramage, Segal, and Seth]{DBLP:journals/corr/BonawitzIKMMPRS16}
Bonawitz, K.~A., Ivanov, V., Kreuter, B., Marcedone, A., McMahan, H.~B., Patel,
  S., Ramage, D., Segal, A., and Seth, K.
\newblock Practical secure aggregation for federated learning on user-held
  data.
\newblock \emph{arXiv:1611.04482}, 2016.

\bibitem[Bonawitz et~al.(2019)Bonawitz, Eichner, Grieskamp, Huba, Ingerman,
  Ivanov, Kiddon, Kone{\v{c}}n{\'y}, Mazzocchi, McMahan, Overveldt, Petrou,
  Ramage, and Roselander]{DBLP:conf/mlsys/BonawitzEGHIIKK19}
Bonawitz, K.~A., Eichner, H., Grieskamp, W., Huba, D., Ingerman, A., Ivanov,
  V., Kiddon, C., Kone{\v{c}}n{\'y}, J., Mazzocchi, S., McMahan, B., Overveldt,
  T.~V., Petrou, D., Ramage, D., and Roselander, J.
\newblock Towards federated learning at scale: System design.
\newblock \emph{Proc. of Machine Learn. Sys.}, 2019.

\bibitem[Bottou et~al.(2018)Bottou, Curtis, and
  Nocedal]{bottou2018optimization}
Bottou, L., Curtis, F.~E., and Nocedal, J.
\newblock Optimization methods for large-scale machine learning.
\newblock \emph{{SIAM} Review}, 60\penalty0 (2):\penalty0 223--311, 2018.

\bibitem[{\c{C}}atak(2015)]{VFL_SMC_2015}
{\c{C}}atak, F.~{\"{O}}.
\newblock Secure multi-party computation based privacy preserving extreme
  learning machine algorithm over vertically distributed data.
\newblock \emph{Proc. Adv. Neural Inf. Process. Syst.}, 9490:\penalty0
  337--345, 2015.

\bibitem[Ceballos et~al.(2020)Ceballos, Sharma, Mugica, Singh, Roman,
  Vepakomma, and Raskar]{SplitNN}
Ceballos, I., Sharma, V., Mugica, E., Singh, A., Roman, A., Vepakomma, P., and
  Raskar, R.
\newblock Splitnn-driven vertical partitioning.
\newblock \emph{arXiv:2008.04137}, 2020.

\bibitem[Cha et~al.(2021)Cha, Sung, and Park]{verticalAutoencoders}
Cha, D., Sung, M., and Park, Y.-R.
\newblock Implementing vertical federated learning using autoencoders:
  Practical application, generalizability, and utility study.
\newblock \emph{JMIR Medical Informatics}, 9\penalty0 (6):\penalty0 e26598,
  2021.

\bibitem[Chen et~al.(2020)Chen, Jin, Sun, and Yin]{chen2020vafl}
Chen, T., Jin, X., Sun, Y., and Yin, W.
\newblock {VAFL:} a method of vertical asynchronous federated learning.
\newblock \emph{arXiv:2007.06081}, 2020.

\bibitem[Cheng et~al.(2021)Cheng, Fan, Jin, Liu, Chen, Papadopoulos, and
  Yang]{SecureBoost}
Cheng, K., Fan, T., Jin, Y., Liu, Y., Chen, T., Papadopoulos, D., and Yang, Q.
\newblock Secureboost: {A} lossless federated learning framework.
\newblock \emph{{IEEE} Intell. Syst.}, 36\penalty0 (6):\penalty0 87--98, 2021.

\bibitem[Das \& Patterson(2021)Das and Patterson]{9415026}
Das, A. and Patterson, S.
\newblock Multi-tier federated learning for vertically partitioned data.
\newblock \emph{Proc. {IEEE} Int. Conf. on Acoust., Speech, and Signal
  Process.}, pp.\  3100--3104, 2021.

\bibitem[Deng et~al.(2009)Deng, Dong, Socher, Li, Li, and
  Fei-Fei]{deng2009imagenet}
Deng, J., Dong, W., Socher, R., Li, L.-J., Li, K., and Fei-Fei, L.
\newblock Imagenet: A large-scale hierarchical image database.
\newblock In \emph{Proc. {IEEE} Conf. Comput. Vis. Pattern Recognit.}, 2009.

\bibitem[Feng \& Yu(2020)Feng and Yu]{DBLP:journals/corr/abs-2001-11154}
Feng, S. and Yu, H.
\newblock Multi-participant multi-class vertical federated learning.
\newblock \emph{arXiv:2001.11154}, 2020.

\bibitem[Geiping et~al.(2020)Geiping, Bauermeister, Dr{\"{o}}ge, and
  Moeller]{DBLP:conf/nips/GeipingBD020}
Geiping, J., Bauermeister, H., Dr{\"{o}}ge, H., and Moeller, M.
\newblock Inverting gradients - how easy is it to break privacy in federated
  learning?
\newblock \emph{Adv. Neural Inf. Process. Syst.}, 2020.

\bibitem[Gu et~al.(2021)Gu, Xu, Huo, Deng, and Huang]{gu2021privacy}
Gu, B., Xu, A., Huo, Z., Deng, C., and Huang, H.
\newblock Privacy-preserving asynchronous vertical federated learning
  algorithms for multiparty collaborative learning.
\newblock \emph{{IEEE} Trans. on Neural Netw. Learn. Syst.}, pp.\  1--13, 2021.
\newblock \doi{10.1109/TNNLS.2021.3072238}.

\bibitem[Gu et~al.(2019)Gu, Lyu, Sun, Li, Chen, Li, and Marsic]{GuLSLCLM19}
Gu, Y., Lyu, X., Sun, W., Li, W., Chen, S., Li, X., and Marsic, I.
\newblock Mutual correlation attentive factors in dyadic fusion networks for
  speech emotion recognition.
\newblock \emph{Proc. {ACM} Int. Conf. on Multimedia}, 2019.

\bibitem[Gupta \& Raskar(2018)Gupta and Raskar]{DBLP:journals/jnca/GuptaR18}
Gupta, O. and Raskar, R.
\newblock Distributed learning of deep neural network over multiple agents.
\newblock \emph{J. Netw. Comput. Appl.}, 116:\penalty0 1--8, 2018.

\bibitem[Han et~al.(2021{\natexlab{a}})Han, Bhatti, Lee, and
  Moon]{han2021accelerating}
Han, D.-J., Bhatti, H.~I., Lee, J., and Moon, J.
\newblock Accelerating federated learning with split learning on locally
  generated losses.
\newblock In \emph{ICML 2021 Workshop on Federated Learning for User Privacy
  and Data Confidentiality}, 2021{\natexlab{a}}.

\bibitem[Han et~al.(2021{\natexlab{b}})Han, Chen, and Poria]{han2021improving}
Han, W., Chen, H., and Poria, S.
\newblock Improving multimodal fusion with hierarchical mutual information
  maximization for multimodal sentiment analysis.
\newblock \emph{Proc. 2020 Conf. Empir. Methods in Nat. Lang. Process.}, pp.\
  9180--9192, 2021{\natexlab{b}}.

\bibitem[Hardy et~al.(2017)Hardy, Henecka, Ivey{-}Law, Nock, Patrini, Smith,
  and Thorne]{hardy2017private}
Hardy, S., Henecka, W., Ivey{-}Law, H., Nock, R., Patrini, G., Smith, G., and
  Thorne, B.
\newblock Private federated learning on vertically partitioned data via entity
  resolution and additively homomorphic encryption.
\newblock \emph{arXiv:1711.10677}, 2017.

\bibitem[He et~al.(2020)He, Annavaram, and Avestimehr]{He2020SplitFed}
He, C., Annavaram, M., and Avestimehr, S.
\newblock Group knowledge transfer: Federated learning of large cnns at the
  edge.
\newblock \emph{Proc. Adv. Neural Inf. Process. Syst.}, 2020.

\bibitem[Hu et~al.(2019)Hu, Niu, Yang, and Zhou]{FDML}
Hu, Y., Niu, D., Yang, J., and Zhou, S.
\newblock {FDML:} {A} collaborative machine learning framework for distributed
  features.
\newblock \emph{Proc. {ACM} Int. Conf. Knowl. Discov. Data Min.}, pp.\
  2232--2240, 2019.

\bibitem[Johnson et~al.(2016)Johnson, Pollard, Shen, Lehman, Feng, Ghassemi,
  Moody, Szolovits, Anthony~Celi, and Mark]{johnson2016mimic}
Johnson, A.~E., Pollard, T.~J., Shen, L., Lehman, L.-w.~H., Feng, M., Ghassemi,
  M., Moody, B., Szolovits, P., Anthony~Celi, L., and Mark, R.~G.
\newblock {MIMIC-III}, a freely accessible critical care database.
\newblock \emph{Nature}, 2016.

\bibitem[Kairouz et~al.(2021)Kairouz, McMahan, Avent, Bellet, Bennis, Bhagoji,
  Bonawitz, Charles, Cormode, Cummings, D'Oliveira, Eichner, Rouayheb, Evans,
  Gardner, Garrett, Gasc{\'{o}}n, Ghazi, Gibbons, Gruteser, Harchaoui, He, He,
  Huo, Hutchinson, Hsu, Jaggi, Javidi, Joshi, Khodak, Kone{\v{c}}n{\'y},
  Korolova, Koushanfar, Koyejo, Lepoint, Liu, Mittal, Mohri, Nock,
  {\"{O}}zg{\"{u}}r, Pagh, Qi, Ramage, Raskar, Raykova, Song, Song, Stich, Sun,
  Suresh, Tram{\`{e}}r, Vepakomma, Wang, Xiong, Xu, Yang, Yu, Yu, and
  Zhao]{kairouz2019advances}
Kairouz, P., McMahan, H.~B., Avent, B., Bellet, A., Bennis, M., Bhagoji, A.~N.,
  Bonawitz, K.~A., Charles, Z., Cormode, G., Cummings, R., D'Oliveira, R.
  G.~L., Eichner, H., Rouayheb, S.~E., Evans, D., Gardner, J., Garrett, Z.,
  Gasc{\'{o}}n, A., Ghazi, B., Gibbons, P.~B., Gruteser, M., Harchaoui, Z., He,
  C., He, L., Huo, Z., Hutchinson, B., Hsu, J., Jaggi, M., Javidi, T., Joshi,
  G., Khodak, M., Kone{\v{c}}n{\'y}, J., Korolova, A., Koushanfar, F., Koyejo,
  S., Lepoint, T., Liu, Y., Mittal, P., Mohri, M., Nock, R., {\"{O}}zg{\"{u}}r,
  A., Pagh, R., Qi, H., Ramage, D., Raskar, R., Raykova, M., Song, D., Song,
  W., Stich, S.~U., Sun, Z., Suresh, A.~T., Tram{\`{e}}r, F., Vepakomma, P.,
  Wang, J., Xiong, L., Xu, Z., Yang, Q., Yu, F.~X., Yu, H., and Zhao, S.
\newblock Advances and open problems in federated learning.
\newblock \emph{Found. Trends Mach. Learn.}, 14\penalty0 (1-2):\penalty0
  1--210, 2021.
\newblock \doi{10.1561/2200000083}.

\bibitem[Karimireddy et~al.(2019)Karimireddy, Rebjock, Stich, and
  Jaggi]{DBLP:conf/icml/KarimireddyRSJ19}
Karimireddy, S.~P., Rebjock, Q., Stich, S.~U., and Jaggi, M.
\newblock Error feedback fixes sign{SGD} and other gradient compression
  schemes.
\newblock \emph{Proc. Int. Conf. on Machine Learn.}, 2019.

\bibitem[Koehrsen(2018)]{book_embeddings_git}
Koehrsen, W.
\newblock Book recommendation system.
\newblock
  https://github.com/WillKoehrsen/wikipedia-data-science/blob/master/notebooks/Book
  2018.

\bibitem[Krizhevsky et~al.(2009)Krizhevsky, Hinton,
  et~al.]{krizhevsky2009learning}
Krizhevsky, A., Hinton, G., et~al.
\newblock Learning multiple layers of features from tiny images.
\newblock 2009.

\bibitem[Li et~al.(2020)Li, Sahu, Zaheer, Sanjabi, Talwalkar, and
  Smith]{li2018federated}
Li, T., Sahu, A.~K., Zaheer, M., Sanjabi, M., Talwalkar, A., and Smith, V.
\newblock Federated optimization in heterogeneous networks.
\newblock \emph{Proc. of Machine Learn. Sys.}, 2020.

\bibitem[Lim et~al.(2020)Lim, Luong, Hoang, Jiao, Liang, Yang, Niyato, and
  Miao]{DBLP:journals/comsur/LimLHJLYNM20}
Lim, W. Y.~B., Luong, N.~C., Hoang, D.~T., Jiao, Y., Liang, Y., Yang, Q.,
  Niyato, D., and Miao, C.
\newblock Federated learning in mobile edge networks: {A} comprehensive survey.
\newblock \emph{{IEEE} Commun. Surveys Tuts.}, 2020.

\bibitem[Lin et~al.(2020)Lin, Stich, Patel, and Jaggi]{lin2018don}
Lin, T., Stich, S.~U., Patel, K.~K., and Jaggi, M.
\newblock Don't use large mini-batches, use local {{SGD}}.
\newblock \emph{Proc. Int. Conf. on Learn. Representations}, 2020.

\bibitem[Lin et~al.(2018)Lin, Han, Mao, Wang, and
  Dally]{DBLP:conf/iclr/LinHM0D18}
Lin, Y., Han, S., Mao, H., Wang, Y., and Dally, B.
\newblock Deep gradient compression: Reducing the communication bandwidth for
  distributed training.
\newblock \emph{Proc. Int. Conf. on Learn. Representations}, 2018.

\bibitem[Liu et~al.(2020)Liu, Zhang, Song, and Letaief]{liu2020client}
Liu, L., Zhang, J., Song, S., and Letaief, K.~B.
\newblock Client-edge-cloud hierarchical federated learning.
\newblock \emph{Proc. {IEEE} Int. Conf. on Comm.}, 2020.

\bibitem[Liu et~al.(2019)Liu, Kang, Zhang, Li, Cheng, Chen, Hong, and
  Yang]{liu2019communication}
Liu, Y., Kang, Y., Zhang, X., Li, L., Cheng, Y., Chen, T., Hong, M., and Yang,
  Q.
\newblock A communication efficient vertical federated learning framework.
\newblock \emph{Adv. Neural Inf. Process. Syst., Workshop on Federated Learning
  for Data Privacy and Confidentiality}, 2019.

\bibitem[Mahendran \& Vedaldi(2015)Mahendran and Vedaldi]{MahendranV15}
Mahendran, A. and Vedaldi, A.
\newblock Understanding deep image representations by inverting them.
\newblock \emph{Proc. {IEEE} Int. Conf. Comput. Vis.}, pp.\  5188--5196, 2015.

\bibitem[McMahan et~al.(2017)McMahan, Moore, Ramage, Hampson, and
  y~Arcas]{pmlr-v54-mcmahan17a}
McMahan, B., Moore, E., Ramage, D., Hampson, S., and y~Arcas, B.~A.
\newblock Communication-efficient learning of deep networks from decentralized
  data.
\newblock \emph{Proc. 20th Int. Conf. on Artif. Intell.}, pp.\  1273--1282,
  2017.

\bibitem[Moritz et~al.(2016)Moritz, Nishihara, Stoica, and
  Jordan]{moritz2016sparknet}
Moritz, P., Nishihara, R., Stoica, I., and Jordan, M.~I.
\newblock Sparknet: Training deep networks in spark.
\newblock \emph{Proc. Int. Conf. on Learn. Representations}, 2016.

\bibitem[Nguyen et~al.(2018)Nguyen, Nguyen, van Dijk, Richt{\'{a}}rik,
  Scheinberg, and Tak{\'{a}}c]{HogWild18}
Nguyen, L.~M., Nguyen, P.~H., van Dijk, M., Richt{\'{a}}rik, P., Scheinberg,
  K., and Tak{\'{a}}c, M.
\newblock {{SGD}} and {Hogwild!} convergence without the bounded gradients
  assumption.
\newblock \emph{Proc. Int. Conf. on Machine Learn.}, 80:\penalty0 3747--3755,
  2018.

\bibitem[Nie et~al.(2021)Nie, Liang, Wang, Wei, and Su]{NieLWWS21}
Nie, W., Liang, Q., Wang, Y., Wei, X., and Su, Y.
\newblock {MMFN:} multimodal information fusion networks for 3d model
  classification and retrieval.
\newblock \emph{{ACM} Trans. on Multimedia Computing, Communications, and
  Applications}, 2021.

\bibitem[Phong et~al.(2018)Phong, Aono, Hayashi, Wang, and
  Moriai]{DBLP:journals/tifs/PhongAHWM18}
Phong, L.~T., Aono, Y., Hayashi, T., Wang, L., and Moriai, S.
\newblock Privacy-preserving deep learning via additively homomorphic
  encryption.
\newblock \emph{{IEEE} Trans. Inf. Forensics Security}, 13\penalty0
  (5):\penalty0 1333--1345, 2018.

\bibitem[Richt{\'{a}}rik \& Tak{\'{a}}c(2016)Richt{\'{a}}rik and
  Tak{\'{a}}c]{DBLP:journals/mp/RichtarikT16}
Richt{\'{a}}rik, P. and Tak{\'{a}}c, M.
\newblock Parallel coordinate descent methods for big data optimization.
\newblock \emph{Math. Program.}, 156\penalty0 (1-2):\penalty0 433--484, 2016.

\bibitem[Rieke et~al.(2020)Rieke, Hancox, Li, Milletarì, Roth, Albarqouni,
  Bakas, Galtier, Landman, Maier-Hein, Ourselin, Sheller, Summers, Trask, Xu,
  Baust, and Cardoso]{DigitalHealthFL}
Rieke, N., Hancox, J., Li, W., Milletarì, F., Roth, H.~R., Albarqouni, S.,
  Bakas, S., Galtier, M.~N., Landman, B.~A., Maier-Hein, K., Ourselin, S.,
  Sheller, M., Summers, R.~M., Trask, A., Xu, D., Baust, M., and Cardoso, M.~J.
\newblock \emph{Digital Medicine}, 2020.

\bibitem[Romanini et~al.(2021)Romanini, Hall, Papadopoulos, Titcombe, Ismail,
  Cebere, Sandmann, Roehm, and Hoeh]{pyvertical}
Romanini, D., Hall, A.~J., Papadopoulos, P., Titcombe, T., Ismail, A., Cebere,
  T., Sandmann, R., Roehm, R., and Hoeh, M.~A.
\newblock {PyVertical}: {A} vertical federated learning framework for
  multi-headed {SplitNN}.
\newblock \emph{Int. Conf. Learn. Representations, Workshop on Distributed and
  Private Machine Learn.}, 2021.

\bibitem[Shi et~al.(2019)Shi, Zhao, Wang, Tang, and
  Chu]{DBLP:conf/ijcai/ShiZWTC19}
Shi, S., Zhao, K., Wang, Q., Tang, Z., and Chu, X.
\newblock A convergence analysis of distributed {{SGD}} with
  communication-efficient gradient sparsification.
\newblock \emph{Proc. Int. Joint Conf. on Artif. Intell.}, 2019.

\bibitem[Shlezinger et~al.(2021)Shlezinger, Chen, Eldar, Poor, and
  Cui]{DBLP:journals/tsp/ShlezingerCEPC21}
Shlezinger, N., Chen, M., Eldar, Y.~C., Poor, H.~V., and Cui, S.
\newblock Uveqfed: Universal vector quantization for federated learning.
\newblock \emph{{IEEE} Trans. Signal Process.}, 69:\penalty0 500--514, 2021.

\bibitem[Stich et~al.(2018)Stich, Cordonnier, and
  Jaggi]{DBLP:conf/nips/StichCJ18}
Stich, S.~U., Cordonnier, J., and Jaggi, M.
\newblock Sparsified {{SGD}} with memory.
\newblock \emph{Adv. Neural Inf. Process. Syst.}, 2018.

\bibitem[Tsitsiklis et~al.(1986)Tsitsiklis, Bertsekas, and
  Athans]{tsitsiklis1986distributed}
Tsitsiklis, J., Bertsekas, D., and Athans, M.
\newblock Distributed asynchronous deterministic and stochastic gradient
  optimization algorithms.
\newblock \emph{{IEEE} Trans. Autom. Control}, 31\penalty0 (9):\penalty0
  803--812, 1986.

\bibitem[Wang et~al.(2019)Wang, Tuor, Salonidis, Leung, Makaya, He, and
  Chan]{wang2019adaptive}
Wang, S., Tuor, T., Salonidis, T., Leung, K.~K., Makaya, C., He, T., and Chan,
  K.
\newblock Adaptive federated learning in resource constrained edge computing
  systems.
\newblock \emph{{IEEE} J. Sel. Areas Commun.}, 37\penalty0 (6):\penalty0
  1205--1221, 2019.

\bibitem[Wannamaker(1997)]{wannamaker1997theory}
Wannamaker, R.~A.
\newblock \emph{The Theory of Dithered Quantization}.
\newblock PhD thesis, 1997.

\bibitem[Wen et~al.(2017)Wen, Xu, Yan, Wu, Wang, Chen, and
  Li]{DBLP:conf/nips/WenXYWWCL17}
Wen, W., Xu, C., Yan, F., Wu, C., Wang, Y., Chen, Y., and Li, H.
\newblock Terngrad: Ternary gradients to reduce communication in distributed
  deep learning.
\newblock \emph{Adv. Neural Inf. Process. Syst.}, 2017.

\bibitem[Woods(2006)]{woods2006multidimensional}
Woods, J.~W.
\newblock \emph{Multidimensional signal, image, and video processing and
  coding}.
\newblock Elsevier, 2006.

\bibitem[Wu et~al.(2015)Wu, Song, Khosla, Yu, Zhang, Tang, and Xiao]{wu20153d}
Wu, Z., Song, S., Khosla, A., Yu, F., Zhang, L., Tang, X., and Xiao, J.
\newblock {3D} shapenets: A deep representation for volumetric shapes.
\newblock \emph{Proc. {IEEE} Int. Conf. Comput. Vis.}, pp.\  1912--1920, 2015.

\bibitem[Yang et~al.(2019)Yang, Liu, Chen, and
  Tong]{DBLP:journals/tist/YangLCT19}
Yang, Q., Liu, Y., Chen, T., and Tong, Y.
\newblock Federated machine learning: Concept and applications.
\newblock \emph{{ACM} Trans. Intell. Syst. Technol.}, 10\penalty0 (2):\penalty0
  12:1--12:19, 2019.

\bibitem[Zamir \& Feder(1996)Zamir and Feder]{DBLP:journals/tit/ZamirF96}
Zamir, R. and Feder, M.
\newblock On lattice quantization noise.
\newblock \emph{{IEEE} Trans. Inf. Theory}, 42\penalty0 (4):\penalty0
  1152--1159, 1996.

\bibitem[Zhang et~al.(2020)Zhang, Yin, Hong, and
  Chen]{DBLP:journals/corr/abs-2012-12420}
Zhang, X., Yin, W., Hong, M., and Chen, T.
\newblock Hybrid federated learning: Algorithms and implementation.
\newblock \emph{arXiv:2012.12420}, 2020.

\bibitem[Zheng et~al.(2022)Zheng, Chen, Zheng, and Zhu]{Zheng2022Secret}
Zheng, F., Chen, C., Zheng, X., and Zhu, M.
\newblock Towards secure and practical machine learning via secret sharing and
  random permutation.
\newblock \emph{Knowl. Based Syst.}, 245:\penalty0 108609, 2022.

\end{thebibliography}
\bibliographystyle{icml2022}

\clearpage

\setcounter{section}{0}
\renewcommand\thesection{\Alph{section}}
\numberwithin{equation}{section}
\counterwithin{figure}{section}
\counterwithin{table}{section}

\newpage
\appendix
\onecolumn

\section{Proofs of Theorems \ref{main.thm} and \ref{main2.thm}} \label{proofs.sec}
In this section, we provide the proofs for Theorems \ref{main.thm} and \ref{main2.thm}.

\subsection{Additional Notation} \label{notation.sec}
Before starting the proofs, we define some additional notation to be
used throughout.
At each iteration $t$, each party $m$ trains with the embeddings
$\hat{\Phi}^t_m$. This is equivalent to the party
training directly with the models $\theta_m^t$ and 
$\theta_j^{t_0}$ for all $j \neq m$,
where $t_0$ is the last communication iteration when
party $m$ received the embeddings. 
We define:
\begin{align}
    \gamma_{m,j}^t = 
    \begin{cases}
        \theta_j^t & m = j \\
        \theta_j^{t_0} & \text{otherwise}
    \end{cases}
    \label{gamma_orig.eq}
\end{align}
to represent party $m$'s view of party $j$'s model at iteration $t$.
We define the column vector  
${\Gamma_m^t = [(\gamma_{m,0}^t)^T;\ldots;(\gamma_{m,M}^t)^T]^T}$
to be party $m$'s view of the global model at iteration $t$.  

We introduce some notation to help with bounding the error
introduced by compression.
We define $\hat{F}_{\B}(\Gamma_m^t)$ to be the stochastic 
loss with compression error for a randomly selected mini-batch $\B$ 
calculated by party $m$ at iteration $t$: 
\begin{align}
    \hat{F}_{\B}(\Gamma_m^t) \coloneqq F_{\B}\left(\theta_0^{t_0} + \epsilon_0^{t_0}, h_1(\theta_1^{t_0}; \X_1^{\B^{t_0}}) + \epsilon_1^{t_0}, \ldots,  h_m(\theta_m^t; \X_m^{\B^{t_0}}), \ldots, h_M(\theta_M^{t_0}; \X_M^{\B^{t_0}}) + \epsilon_M^{t_0}\right). 
\end{align}

Recall the recursion over the global model $\Theta$:
\begin{align}
    \Theta^{t+1} = \Theta^t - \eta^{t_0} \hat{\G}^t.
\end{align}
We can equivalently define $\hat{\G}^t$ as follows: 
\begin{align}
    \hat{\G}^t = \left[(\nabla_0 \hat{F}_{\B}(\Gamma_0^t))^T, \ldots, 
    (\nabla_M \hat{F}_{\B}(\Gamma_M^t))^T\right]^T.
\end{align}

Note that the compression error in $\hat{F}(\cdot)$ is applied to the embeddings, 
and not the model parameters. 
Thus, $F(\cdot)$ and $\hat{F}(\cdot)$ are different functions.
In several parts of the proof, we need to bound the compression
error in $\nabla_m \hat{F}_{\B}(\Gamma_m^t)$.

For our analysis, we redefine the set of embeddings for a mini-batch $\B$
of size $B$ from party $m$ as a matrix:
\begin{align}
    h_m(\theta_m; \X_m^{\B}) &\coloneqq \left[h_m(\theta_m; x^{\B^1}_m), \ldots , h_m(\theta_m; x^{\B^B}_m) \right]. 
\end{align}
$h_m(\theta_m ; \X_m^{\B})$ is a matrix with dimensions $P_m \times B$
where each column
is the embedding from party $m$ 
for a single sample in the mini-batch. 

Let $P = \sum_{m=0}^M P_m$ be the sum of the sizes of all embeddings.
We redefine the set of embeddings used by a party $m$ to calculate its gradient
without compression error as a matrix:
\begin{align}
    \hat{\Phi}_m^t = \left[(\theta_0^{t_0})^T,(h_1(\theta_1^{t_0}; \X_1^{\B^{t_0}}))^T, \ldots, 
    (h_m(\theta_m^t; \X_m^{\B^{t_0}}))^T, \ldots, (h_M(\theta_M^{t_0}; \X_M^{\B^{t_0}}))^T\right]^T.
\end{align}
$\hat{\Phi}_m^t$ is a matrix with dimensions $P \times B$ 
where each column
is the concatenation of embeddings for all parties 
for a single sample in the mini-batch. 

Recall the set of compression error vectors for a mini-batch $\B$
of size $B$ from party $m$ is the matrix:
\begin{align}
    \epsilon_m^{t_0} &\coloneqq \left[\epsilon_m^{\B^1}, \ldots , \epsilon_m^{\B^B} \right].
\end{align}
$\epsilon_m^{t_0}$ is a matrix of dimensions $P_m \times B$
where each column
is the compression error from party $m$ 
for a single sample in the mini-batch. 

We define the compression error on each embedding used in
party $m$'s gradient calculation at iteration $t$:
\begin{align}
    E_m^{t_0} = \left[(\epsilon_0^{t_0})^T,\ldots,(\epsilon_{m-1}^{t_0})^T, \zero^T, 
    (\epsilon_{m-1}^{t_0})^T, \ldots, (\epsilon_{M}^{t_0})^T\right]^T.
\end{align}
$E_m^{t_0}$ is a matrix with dimensions $P \times B$
where each column
is the concatenation of compression error on embeddings for all parties 
for a single sample in the mini-batch. 

With some abuse of notation, we define:
\begin{align}
    \nabla_m F_{\B}(\Phi_m^t + E_m^{t_0}) \coloneqq \nabla_m \hat{F}_{\B}(\Gamma_m^t).
\end{align}

Note that we can apply the chain rule to $\nabla_m \hat{F}_{\B}(\Gamma_m^t)$:
\begin{align}
    \nabla_m \hat{F}_{\B}(\Gamma_m^t) &= \nabla_{\theta_m} h_m(\theta_m^t) \nabla_{h_m(\theta_m)} F_{\B}(\Phi_m^t + E_m^{t_0}). 
        \label{taylor_orig1.eq}
\end{align}
With this expansion, we can now apply Taylor series expansion
to $\nabla_{h_m(\theta_m)} F_{\B}(\Phi_m^t + E_m^{t_0})$
around the point $\Phi_m^t$:
\begin{align}
    \nabla_{h_m(\theta_m)} F_{\B}(\Phi_m^t + E_m^{t_0}) 
    &= \nabla_{h_m(\theta_m)} F_{\B}(\Phi_m^t) 
        + \nabla_{h_m(\theta_m)}^2 F_{\B}(\Phi_m^t)^t E_m^{t_0} + \ldots
        \label{taylor_orig.eq}
\end{align}
We let the infinite sum of all terms in this Taylor series 
from the second partial derivatives and up be denoted as $R_0^m$:
\begin{align}
    R^m_0(\Phi_m^t + E_m^{t_0}) 
    \coloneqq \nabla_{h_m(\theta_m)}^2 F_{\B}(\Phi_m^t)^T E_m^{t_0} + \ldots
\end{align}
Note that all compression error is in
$R^m_0(\Phi_m^t + E_m^{t_0})$. Presented in Section~\ref{lemmas.sec}, 
the proof of Lemma~\ref{error'.lemma} shows how we can bound $R^m_0(\Phi_m^t + E_m^{t_0})$, 
bounding the compression error in $\nabla_m \hat{F}_{\B}(\Gamma_m^t)$.

Let $\mathbb{E}^{t_0} = \mathbb{E}_{\B^{t_0}}
[~ \cdot ~|~ \{\Theta^{\tau}\}_{\tau=0}^{t_0}]$.
Note that by Assumption~\ref{bias.assum}, 
$\Ebatch{\G^{t_0}} = \nabla F(\Theta^{t_0})$
as when there is no compression error in the gradients $\G$,
they are equal to the full-batch gradient in expectation when conditioned
on the model parameters up to the iteration $t_0$.
However, this is not true for iterations $t_0+1 \leq t \leq t_0+Q-1$,
as we reuse the mini-batch $\B^{t_0}$ in these local iterations.
We upper bound the error introduced by stochastic gradients calculated
during local iterations in Lemma~\ref{diff.lemma}.

\subsection{Supporting Lemmas} \label{lemmas.sec} 
Next, we provide supporting lemmas and their proofs.

We restate Lemma~\ref{error.lemma} here:
\begin{lemma}{1}
    Under Assumptions~\ref{smooth2.assum}-\ref{bounded.assum}, the norm of the difference
    between the objective function value with and without error
    is bounded: 
\begin{align}
    \mathbb{E} \lrVert{\nabla_m F_{\B}(\hat{\Phi}_m^t) - \nabla_m F_{\B}(\Phi_m^t)}^2
    &\leq H_m^2 G_m^2 \sum_{j=0, j \neq m}^M \mathcal{E}_j^{t_0}.
\end{align}
\end{lemma}

To prove Lemma~\ref{error.lemma}, we first prove the following lemma:
\begin{lemma}{1'}\label{error'.lemma}
    Under Assumptions~\ref{smooth2.assum}-\ref{bounded.assum}, 
    the squared norm of the partial derivatives for party $m$'s embedding
    multiplied by the Taylor series terms $R^m_0(\Phi_m^t+E_m^{t_0})$ is bounded:
\begin{align}
    \lrVert{\nabla_{\theta_m} h_m(\theta_m^t) R^m_0(\Phi_m^t+E_m^{t_0})}^2 
    &\leq H_m^2 G_m^2 \lrVert{E_m^{t_0}}_{\mathcal{F}}.
\end{align}
\end{lemma}

\begin{proof}
\begin{align}
    \lrVert{\nabla_{\theta_m} h_m(\theta_m^t) R^m_0(\Phi_m^t+E_m^{t_0})}^2
    &\leq \lrVert{\nabla_{\theta_m} h_m(\theta_m^t)}_{\mathcal{F}}^2 \lrVert{R^m_0(\Phi_m^t+E_m^{t_0})}_{\mathcal{F}}^2 \\
    &\leq H_m^2 \lrVert{\nabla_{\theta_m} h_m(\theta_m^t)}_{\mathcal{F}}^2 \lrVert{E_m^{t_0}}_{\mathcal{F}}^2 \label{taylor.eq} 
\end{align}
where (\ref{taylor.eq}) follows from Assumption~\ref{smooth2.assum} and
the following property of the Taylor series approximation error: 
\begin{align}
    \lrVert{R_0^m(\Phi_m^t + E_m^{t_0})}_{\mathcal{F}} \leq H_m \lrVert{E_m^{t_0}}_{\mathcal{F}}.
\end{align}
 
Applying Assumption~\ref{bounded.assum}, we have: 
\begin{align}
    \lrVert{\nabla_{\theta_m} h_m(\theta_m^t) R^m_0(\Phi_m^t+E_m^{t_0})}^2
    &\leq H_m^2 G_m^2 \lrVert{E_m^{t_0}}_{\mathcal{F}}^2.  \label{assum5.eq}
\end{align}
\end{proof}

We now prove Lemma~\ref{error.lemma}.
\begin{proof}
Recall that:
\begin{align}
    \nabla_m \hat{F}_{\B}(\Gamma_m^t) &= \nabla_m F_{\B}(\Phi_m^t + E_m^{t_0}) \\
    &= \nabla_{\theta_m} h_m(\theta_m^t) \nabla_{h_m(\theta_m)} F_{\B}(\Phi_m^t + E_m^{t_0}).
\end{align}

Next we apply Taylor series expansion as in (\ref{taylor_orig.eq}): 
\begin{align}
    \nabla_m \hat{F}_{\B}(\Gamma_m^t)
    &= \nabla_{\theta_m} h_m(\theta_m^t) \left(\nabla_{h_m(\theta_m)} F_{\B}(\Phi_m^t) + R^m_0(\Phi_m^t + E_m^{t_0})\right) \\
    &= \nabla_m F_{\B}(\Gamma_m^t) + \nabla_{\theta_m} h_m(\theta_m^t) R^m_0(\Phi_m^t + E_m^{t_0}) 
\end{align}
Rearranging and applying expectation and the squared 2-norm, we can bound further: 
\begin{align}
    \mathbb{E} \lrVert{\nabla_m \hat{F}_{\B}(\Gamma_m^t) - \nabla_m F_{\B}(\Gamma_m^t)}^2
    &= \mathbb{E} \lrVert{\nabla_{\theta_m} h_m(\theta_m^t) R^m_0(\Phi_m^t+E_m^{t_0})}^2 \\
    &\leq H_m^2 G_m^2 \mathbb{E} \lrVert{E_m^{t_0}}_{\mathcal{F}}^2  \label{lemma1'.eq} \\
    &= H_m^2 G_m^2 \sum_{j \neq m} \mathbb{E} \lrVert{\epsilon_j^{t_0}}_{\mathcal{F}}^2 \label{comp_error.eq} \\
    &= H_m^2 G_m^2 \sum_{j \neq m} \mathcal{E}_j^{t_0}  \label{comp_error2.eq}
\end{align}
where (\ref{lemma1'.eq})
follows from Lemma~\ref{error'.lemma},
(\ref{comp_error.eq}) follows from the definition of $E_m^{t_0}$,
and (\ref{comp_error2.eq}) follows from Definition~\ref{compress.assum}.
\end{proof}

\begin{lemma}{2} \label{diff.lemma}
    If $\eta^{t_0} \leq \frac{1}{4Q \max_m L_m}$,
    then under Assumptions~\ref{smooth.assum}-\ref{bounded.assum} we can bound the
    conditional expected squared norm difference of 
    gradients $\G^{t_0}$ and $\hat{\G}^t$ 
    for iterations $t_0$ to $t_0+Q-1$ as follows:
\begin{align}
    \sum_{t=t_0}^{t_0+Q-1} \Ebatch{\lrVert{\hat{\G}^t - \G^{t_0}}^2}
    &\leq
    16Q^3(\eta^{t_0})^2 \sum_{m=0}^M L_m^2\lrVert{\nabla_m F(\Theta^{t_0})}^2 
    \nonumber \\ &~~~
    + 16Q^3(\eta^{t_0})^2\sum_{m=0}^M L_m^2 \frac{\sigma_m^2}{B} 
    \nonumber \\ &~~~
    + 64Q^3\sum_{m=0}^M H_m^2 G_m^2 \lrVert{E_m^{t_0}}_{\mathcal{F}}^2.
\end{align}
\end{lemma}

\begin{proof}
\begin{align}
    \Ebatch{\lrVert{\hat{\G}^t - \G^{t_0}}^2}  
    &= \sum_{m=0}^M \Ebatch{\lrVert{\nabla_m \hat{F}_{\B}(\Gamma_m^t) - \nabla_m F_{\B}(\Gamma_m^{t_0})}^2} \\ 
    &= \sum_{m=0}^M \Ebatch{\lrVert{\nabla_m \hat{F}_{\B}(\Gamma_m^t) - \hat{F}_{\B}(\Gamma_m^{t-1}) 
    + \nabla_m \hat{F}_{\B}(\Gamma_m^{t-1}) - \nabla_m F_{\B}(\Gamma_m^{t_0})}^2} \\
    &\leq \left(1+n\right)\sum_{m=0}^M \Ebatch{\lrVert{\nabla_m \hat{F}_{\B}(\Gamma_m^t) - \nabla_m \hat{F}_{\B}(\Gamma_m^{t-1}) }^2}
    \nonumber \\ &~~~
    + \left(1+\frac{1}{n}\right)\sum_{m=0}^M \Ebatch{\lrVert{\nabla_m \hat{F}_{\B}(\Gamma_m^{t-1}) - \nabla_m F_{\B}(\Gamma_m^{t_0})}^2} \label{nplus.eq} \\
    &\leq 2\left(1+n\right)\sum_{m=0}^M \Ebatch{\lrVert{\nabla_m F_{\B}(\Gamma_m^t) - \nabla_m F_{\B}(\Gamma_m^{t-1}) }^2}
    \nonumber \\ &~~~
    + 2\left(1+n\right)\sum_{m=0}^M \Ebatch{\lrVert{\nabla_{\theta_m} h_m(\theta_m^t) R_0^m(\Phi_m^t+E_m^{t_0}) - \nabla_{\theta_m} h_m(\theta_m^{t-1}) R_0^m(\Phi_m^{t-1}+E_m^{t-1})}^2} 
    \nonumber \\ &~~~
    + \left(1+\frac{1}{n}\right)\sum_{m=0}^M \Ebatch{\lrVert{\nabla_m \hat{F}_{\B}(\Gamma_m^{t-1}) - \nabla_m F_{\B}(\Gamma_m^{t_0})}^2} \\
    &\leq 2\left(1+n\right)\sum_{m=0}^M \Ebatch{\lrVert{\nabla_m F_{\B}(\Gamma_m^t) - \nabla_m F_{\B}(\Gamma_m^{t-1}) }^2}
    + 8\left(1+n\right)\sum_{m=0}^M H_m^2 G_m^2 \lrVert{E_m^{t_0}}^2 
    \nonumber \\ &~~~
    + \left(1+\frac{1}{n}\right)\sum_{m=0}^M \Ebatch{\lrVert{\nabla_m \hat{F}_{\B}(\Gamma_m^{t-1}) - \nabla_m F_{\B}(\Gamma_m^{t_0})}^2} \label{comperror3.eq} 
\end{align}
where (\ref{nplus.eq}) follows from the fact that $(X+Y)^2 \leq (1+n)X^2+(1+\frac{1}{n})Y^2$ for some positive $n$
and (\ref{comperror3.eq}) follows from Lemma~\ref{error'.lemma}.

Applying Assumption~\ref{smooth.assum} to the first term in (\ref{nplus.eq}) we have:
\begin{align}
    \Ebatch{\lrVert{\hat{\G}^t - \G^{t_0}}^2}  
    &\leq 2\left(1+n\right)\sum_{m=0}^M L_m^2\Ebatch{\lrVert{\Gamma_m^t - \Gamma_m^{t-1} }^2} 
    \nonumber \\ &~~~
    + 2\left(1+\frac{1}{n}\right)\sum_{m=0}^M \Ebatch{\lrVert{\nabla_m \hat{F}_{\B}(\Gamma_m^{t-1}) - \nabla_m F_{\B}(\Gamma_m^{t_0})}^2}  
    \nonumber \\ &~~~
    + 8\left(1+n\right)\sum_{m=0}^M H_m^2 G_m^2 \lrVert{E_m^{t_0}}^2 \\
    &= 2(\eta^{t_0})^2 \left(1+n\right)\sum_{m=0}^M L_m^2\Ebatch{\lrVert{\nabla_m \hat{F}_{\B}(\Gamma_m^{t-1})}^2} 
    \nonumber \\ &~~~
    + 2\left(1+\frac{1}{n}\right)\sum_{m=0}^M \Ebatch{\lrVert{\nabla_m \hat{F}_{\B}(\Gamma_m^{t-1}) - \nabla_m F_{\B}(\Gamma_m^{t_0})}^2}  
    \nonumber \\ &~~~
    + 8\left(1+n\right)\sum_{m=0}^M H_m^2 G_m^2 \lrVert{E_m^{t_0}}^2 \label{recursion1.eq} 
\end{align}
where (\ref{recursion1.eq}) follows from the update rule 
$\Gamma_m^t = \Gamma_m^{t-1} - \eta^{t_0} \nabla_m \hat{F}_{\B}(\Gamma_m^{t-1})$.

Bounding further:
\begin{align}
    \Ebatch{\lrVert{\hat{\G}^t - \G^{t_0}}^2}  
    &\leq 2(\eta^{t_0})^2 \left(1+n\right)\sum_{m=0}^M L_m^2
    \Ebatch{\lrVert{\nabla_m \hat{F}_{\B}(\Gamma_m^{t-1}) 
    - \nabla_m F_{\B}(\Gamma_m^{t_0}) + \nabla_m F_{\B}(\Gamma_m^{t_0})}^2} 
    \nonumber \\ &~~~~~~~~~~~
    + \left(1+\frac{1}{n}\right)\sum_{m=0}^M \Ebatch{\lrVert{\nabla_m \hat{F}_{\B}(\Gamma_m^{t-1}) - \nabla_m F_{\B}(\Gamma_m^{t_0})}^2}  
    \nonumber \\ &~~~~~~~~~~~
    + 8\left(1+n\right)\sum_{m=0}^M H_m^2 G_m^2 \lrVert{E_m^{t_0}}^2 \\ 
    &~~~\leq 4(\eta^{t_0})^2 \left(1+n\right)\sum_{m=0}^M L_m^2 
    \Ebatch{\lrVert{\nabla_m \hat{F}_{\B}(\Gamma_m^{t-1}) - \nabla_m F_{\B}(\Gamma_m^{t_0})}^2} 
    \nonumber \\ &~~~~~~~~~~~
    + 4(\eta^{t_0})^2 \left(1+n\right)\sum_{m=0}^M L_m^2\Ebatch{\lrVert{\nabla_m F_{\B}(\Gamma_m^{t_0})}^2}  
    \nonumber \\ &~~~~~~~~~~~
    + \left(1+\frac{1}{n}\right)\sum_{m=0}^M \Ebatch{\lrVert{\nabla_m \hat{F}_{\B}(\Gamma_m^{t-1}) - \nabla_m F_{\B}(\Gamma_m^{t_0})}^2}  
    \nonumber \\ &~~~~~~~~~~~
    + 8\left(1+n\right)\sum_{m=0}^M H_m^2 G_m^2 \lrVert{E_m^{t_0}}^2 \\ 
    &~~~= \sum_{m=0}^M \left(4(\eta^{t_0})^2 \left(1+n\right)L_m^2 + \left(1+\frac{1}{n}\right)\right)
    \Ebatch{\lrVert{\nabla_m \hat{F}_{\B}(\Gamma_m^{t-1}) - \nabla_m F_{\B}(\Gamma_m^{t_0})}^2} 
    \nonumber \\ &~~~~~~~~~~~
    + 4(\eta^{t_0})^2 \left(1+n\right)\sum_{m=0}^M L_m^2\Ebatch{\lrVert{\nabla_m F_{\B}(\Gamma_m^{t_0})}^2}  
    \nonumber \\ &~~~~~~~~~~~
    + 8\left(1+n\right)\sum_{m=0}^M H_m^2 G_m^2 \lrVert{E_m^{t_0}}^2. \label{definen.eq}
\end{align}

Let $n=Q$. We simplify (\ref{definen.eq}) further:
\begin{align}
    &\Ebatch{\lrVert{\hat{\G}^t - \G^{t_0}}^2}  
    \nonumber \\ 
    &~~~~~\leq \sum_{m=0}^M \left(4(\eta^{t_0})^2 \left(1+Q\right)L_m^2 + \left(1+\frac{1}{Q}\right)\right)
    \Ebatch{\lrVert{\nabla_m \hat{F}_{\B}(\Gamma_m^{t-1}) - \nabla_m F_{\B}(\Gamma_m^{t_0})}^2} 
    \nonumber \\ &~~~~~~~~~~~~~~~~
    + 4(\eta^{t_0})^2 \left(1+Q\right)\sum_{m=0}^M L_m^2\Ebatch{\lrVert{\nabla_m F_{\B}(\Gamma_m^{t_0})}^2}  
    \nonumber \\ &~~~~~~~~~~~~~~~~
    + 8\left(1+Q\right)\sum_{m=0}^MH_m^2 G_m^2 \lrVert{E_m^{t_0}}_{\mathcal{F}}^2. \label{definedn.eq}
\end{align}

Let $\eta^{t_0} \leq \frac{1}{4Q \max_m L_m}$. We bound (\ref{definedn.eq}) as follows:
\begin{align}
    \Ebatch{\lrVert{\hat{\G}^t - \G^{t_0}}^2}  
    &\leq \left(\frac{\left(1+Q\right)}{4Q^2} + \left(1+\frac{1}{Q}\right)\right)\sum_{m=0}^M 
    \Ebatch{\lrVert{\nabla_m \hat{F}_{\B}(\Gamma_m^{t-1}) - \nabla_m F_{\B}(\Gamma_m^{t_0})}^2} 
    \nonumber \\ &~~~
    + 4(\eta^{t_0})^2 \left(1+Q\right)\sum_{m=0}^M L_m^2\Ebatch{\lrVert{\nabla_m F_{\B}(\Gamma_m^{t_0})}^2}  
    \nonumber \\ &~~~
    + 8(1+Q)\sum_{m=0}^M H_m^2 G_m^2 \lrVert{E_m^{t_0}}_{\mathcal{F}}^2 \\
    &\leq \left(\frac{1}{2Q} + \left(1+\frac{1}{Q}\right)\right)\sum_{m=0}^M 
    \Ebatch{\lrVert{\nabla_m \hat{F}_{\B}(\Gamma_m^{t-1}) - \nabla_m F_{\B}(\Gamma_m^{t_0})}^2} 
    \nonumber \\ &~~~
    + 4(\eta^{t_0})^2 \left(1+Q\right)\sum_{m=0}^M L_m^2\Ebatch{\lrVert{\nabla_m F_{\B}(\Gamma_m^{t_0})}^2}  
    \nonumber \\ &~~~
    + 8(1+Q)\sum_{m=0}^M H_m^2 G_m^2 \lrVert{E_m^{t_0}}_{\mathcal{F}}^2 \\
    &\leq \left(1+\frac{2}{Q}\right)\sum_{m=0}^M 
    \Ebatch{\lrVert{\nabla_m \hat{F}_{\B}(\Gamma_m^{t-1}) - \nabla_m F_{\B}(\Gamma_m^{t_0})}^2} 
    \nonumber \\ &~~~
    + 4(\eta^{t_0})^2 \left(1+Q\right)\sum_{m=0}^M L_m^2\Ebatch{\lrVert{\nabla_m F_{\B}(\Gamma_m^{t_0})}^2}  
    \nonumber \\ &~~~
    + 8(1+Q)\sum_{m=0}^M H_m^2 G_m^2 \lrVert{E_m^{t_0}}_{\mathcal{F}}^2. 
\end{align}

We define the following notation for simplicity: 
\begin{align}
    A^t &\coloneqq \sum_{m=0}^M \Ebatch{\lrVert{\nabla_m \hat{F}_{\B}(\Gamma_m^{t}) - \nabla_m F_{\B}(\Gamma_m^{t_0})}^2} \\
B_0     &\coloneqq 4(\eta^{t_0})^2 \left(1+Q\right)\sum_{m=0}^M L_m^2\Ebatch{\lrVert{\nabla_m F_{\B}(\Gamma_m^{t_0})}^2} \\
B_1     &\coloneqq 8(1+Q)\sum_{m=0}^M H_m^2 G_m^2 \lrVert{E_m^{t_0}}_{\mathcal{F}}^2 \\
      C &\coloneqq \left(1+\frac{2}{Q}\right).
\end{align}

Note that we have shown that $A^t \leq CA^{t-1} + B_0 + B_1$.
Therefore:
\begin{align}
    A^{t_0+1} &\leq CA^{t_0} + (B_0 + B_1) \\
    A^{t_0+2} &\leq C^2A^{t_0} + C(B_0 + B_1) + (B_0 + B_1) \\
    A^{t_0+3} &\leq C^3A^{t_0} + C^2(B_0 + B_1) + C(B_0 + B_1) + (B_0 + B_1) \\
    \vdots \\
    A^t &\leq C^{t-t_0-1}A^{t_0} + (B_0 + B_1) \sum_{k=0}^{t-t_0-2} C^k \\
        &= C^{t-t_0-1}A^{t_0} + (B_0 + B_1) \frac{C^{t-t_0-1} - 1}{C - 1}. \label{recurse2.eq} 
\end{align}

We bound the first term in (\ref{recurse2.eq}) by applying Lemma~\ref{error.lemma}:
\begin{align}
    A^{t_0} &= \sum_{m=0}^M \Ebatch{\lrVert{\nabla_m \hat{F}_{\B}(\Gamma_m^{t_0}) - \nabla_m F_{\B}(\Gamma_m^{t_0})}^2} \\
&\leq \sum_{m=0}^M H_m^2 G_m^2 \lrVert{E_m^{t_0}}_{\mathcal{F}}^2. 
\end{align}
Summing over the set of local iterations $t_0,\ldots,t_0^+$, where
$t_0^+ \coloneqq t_0+Q-1$:
\begin{align}
    \sum_{t=t_0}^{t_0^+} C^{t-t_0-1}A^{t_0} 
    &=  A^{t_0}\sum_{t=t_0}^{t_0^+} C^{t-t_0-1} \\
    &=  A^{t_0}\frac{C^Q-1}{C-1} \\
    &=  A^{t_0}\frac{\left(1+\frac{2}{Q}\right)^Q-1}{\left(1+\frac{2}{Q}\right)-1} \\
    &\leq  QA^{t_0}\frac{e^2-1}{2} \\
    &\leq  4QA^{t_0} \\
    &\leq 4Q\sum_{m=0}^M H_m^2 G_m^2 \lrVert{E_m^{t_0}}_{\mathcal{F}}^2. 
\end{align}

It is left to bound the second term in (\ref{recurse2.eq}) 
over the set of local iterations $t_0,\ldots,t_0+Q-1$.
\begin{align}
    \sum_{t=t_0}^{t_0^+} (B_0 + B_1) \frac{C^{t-t_0-1} - 1}{C - 1} 
    &\leq \sum_{t=t_0}^{t_0^+} (B_0 + B_1) \frac{C^{t-t_0-1} - 1}{C - 1} \\
    &= \frac{(B_0 + B_1)}{C - 1} \left(\sum_{t=t_0}^{t_0^+} C^{t-t_0-1} - Q\right)\\
    &= \frac{(B_0 + B_1)}{C - 1} \left(\frac{C^Q-1}{C-1} - Q\right)\\
    &= \frac{(B_0 + B_1)}{\left(1+\frac{2}{Q}\right) - 1} \left(\frac{\left(1+\frac{2}{Q}\right)^Q-1}{\left(1+\frac{2}{Q}\right)-1} - Q\right)\\
    &= \frac{Q(B_0 + B_1)}{2} \left(\frac{Q\left[\left(1+\frac{2}{Q}\right)^Q-1\right]}{2} - Q\right)\\
    &= \frac{Q^2(B_0 + B_1)}{2} \left(\frac{\left(1+\frac{2}{Q}\right)^Q-1}{2} - 1\right)\\
    &\leq \frac{Q^2(B_0 + B_1)}{2} \left(\frac{e^2-1}{2} - 1\right)\\
    &\leq 2Q^2(B_0 + B_1) 
\end{align}

Plugging the values for $B_0$ and $B_1$:
\begin{align}
    \sum_{t=t_0}^{t_0^+} (B_0 + B_1) \frac{C^{t-t_0-1} - 1}{C - 1} 
    &\leq 
    8Q^2(\eta^{t_0})^2 \left(1+Q\right)\sum_{m=0}^M L_m^2\Ebatch{\lrVert{\nabla_m F_{\B}(\Gamma_m^{t_0})}^2} 
    \nonumber \\ &~~~
    + 16Q^2(1+Q)\sum_{m=0}^M H_m^2 G_m^2 \lrVert{E_m^{t_0}}_{\mathcal{F}}^2 
\end{align}

Applying Assumption~\ref{var.assum} and adding in the first term in (\ref{recurse2.eq}):
\begin{align}
    \sum_{t=t_0}^{t_0^+} A^t
    &\leq
    8Q^2(\eta^{t_0})^2 \left(1+Q\right)\sum_{m=0}^M L_m^2\lrVert{\nabla_m F(\Theta^{t_0})}^2 
    \nonumber \\ &~~~
    + 8Q^2(\eta^{t_0})^2 \left(1+Q\right)\sum_{m=0}^M L_m^2 \frac{\sigma_m^2}{B} 
    \nonumber \\ &~~~
    + 4(4Q^2(1+Q) + Q)\sum_{m=0}^M H_m^2 G_m^2 \lrVert{E_m^{t_0}}_{\mathcal{F}}^2 \\
    &\leq
    16Q^3(\eta^{t_0})^2 \sum_{m=0}^M L_m^2\lrVert{\nabla_m F(\Theta^{t_0})}^2 
    \nonumber \\ &~~~
    + 16Q^3(\eta^{t_0})^2\sum_{m=0}^M L_m^2 \frac{\sigma_m^2}{B} 
    \nonumber \\ &~~~
    + 64Q^3\sum_{m=0}^M H_m^2 G_m^2 \lrVert{E_m^{t_0}}_{\mathcal{F}}^2.
\end{align}
\end{proof}

\subsection{Proof of Theorems \ref{main.thm} and \ref{main2.thm}} \label{main_proof.sec}
Let $t_0^+ \coloneqq t_0+Q-1$.
By Assumption~\ref{smooth.assum}:
\begin{align}
    F(\Theta^{t_0^+}) - F(\Theta^{t_0})
    &\leq \left\langle \nabla F(\Theta^{t_0}), \Theta^{t_0^+} - \Theta^{t_0} \right\rangle 
    + \frac{L}{2} \lrVert{\Theta^{t_0^+} - \Theta^{t_0}}^2 \\
    &= -\left\langle \nabla F(\Theta^{t_0}), \sum_{t=t_0}^{t_0^+} \eta^{t_0} \hat{\G}^t \right\rangle 
    + \frac{L}{2} \lrVert{\sum_{t=t_0}^{t_0^+} \eta^{t_0} \hat{\G}^t}^2 \\ 
    &\leq -\sum_{t=t_0}^{t_0^+} \eta^{t_0} \left\langle \nabla F(\Theta^{t_0}), \hat{\G}^t \right\rangle 
    + \frac{LQ}{2} \sum_{t=t_0}^{t_0^+}(\eta^{t_0})^2\lrVert{\hat{\G}^t}^2 \label{cauchy1.eq}
\end{align}
where (\ref{cauchy1.eq}) follows from fact that $(\sum_{n=1}^N x_n)^2 \leq N\sum_{n=1}^N x_n^2$.

We bound further:
\begin{align}
    F(\Theta^{t_0^+}) - F(\Theta^{t_0})
    &\leq -\sum_{t=t_0}^{t_0^+} \eta^{t_0} 
    \left\langle \nabla F(\Theta^{t_0}), \hat{\G}^t - \G^{t_0} \right\rangle 
    -\sum_{t=t_0}^{t_0^+} \eta^{t_0} 
    \left\langle \nabla F(\Theta^{t_0}), \G^{t_0} \right\rangle 
    \nonumber \\ &~~~
    + \frac{LQ}{2} \sum_{t=t_0}^{t_0^+}(\eta^{t_0})^2\lrVert{\hat{\G}^t - \G^{t_0} + \G^{t_0}}^2 \\
    &\leq -\sum_{t=t_0}^{t_0^+} \eta^{t_0} 
    \left\langle \nabla F(\Theta^{t_0}), \hat{\G}^t - \G^{t_0} \right\rangle 
    -\sum_{t=t_0}^{t_0^+} \eta^{t_0} 
    \left\langle \nabla F(\Theta^{t_0}), \G^{t_0} \right\rangle 
    \nonumber \\ &~~~
    + LQ \sum_{t=t_0}^{t_0^+}(\eta^{t_0})^2\lrVert{\hat{\G}^t - \G^{t_0}}^2 
    + LQ \sum_{t=t_0}^{t_0^+}(\eta^{t_0})^2\lrVert{\G^{t_0}}^2 \\ 
    &= \sum_{t=t_0}^{t_0^+} \eta^{t_0} 
    \left\langle -\nabla F(\Theta^{t_0}), \G^{t_0} - \hat{\G}^t \right\rangle 
    -\sum_{t=t_0}^{t_0^+} \eta^{t_0} 
    \left\langle \nabla F(\Theta^{t_0}), \G^{t_0} \right\rangle 
    \nonumber \\ &~~~
    + LQ \sum_{t=t_0}^{t_0^+}(\eta^{t_0})^2\lrVert{\hat{\G}^t - \G^{t_0}}^2 
    + LQ \sum_{t=t_0}^{t_0^+}(\eta^{t_0})^2\lrVert{\G^{t_0}}^2. \\
    &\leq \frac{1}{2}\sum_{t=t_0}^{t_0^+} \eta^{t_0} \lrVert{\nabla F(\Theta^{t_0})}^2 
    \nonumber \\ &~~~
    + \frac{1}{2}\sum_{t=t_0}^{t_0^+} \eta^{t_0} \lrVert{\hat{\G}^t - \G^{t_0}}^2  
    -\sum_{t=t_0}^{t_0^+} \eta^{t_0} 
    \left\langle \nabla F(\Theta^{t_0}), \G^{t_0} \right\rangle 
    \nonumber \\ &~~~
    + LQ \sum_{t=t_0}^{t_0^+}(\eta^{t_0})^2\lrVert{\hat{\G}^t - \G^{t_0}}^2 
    + LQ \sum_{t=t_0}^{t_0^+}(\eta^{t_0})^2\lrVert{\G^{t_0}}^2 \label{product_half.eq} 
\end{align}
where (\ref{product_half.eq}) follows from the fact that 
$A \cdot B = \frac{1}{2}A^2 + \frac{1}{2}B^2 - \frac{1}{2}(A-B)^2$.

We apply the expectation $\mathbb{E}^{t_0}$ to both sides of (\ref{product_half.eq}):
\begin{align}
    \Ebatch{F(\Theta^{t_0^+})} - F(\Theta^{t_0})
    &\leq -\frac{1}{2}\sum_{t=t_0}^{t_0^+} \eta^{t_0} \lrVert{\nabla F(\Theta^{t_0})}^2 
    + \frac{1}{2}\sum_{t=t_0}^{t_0^+} \eta^{t_0}(1+2LQ\eta^{t_0}) \Ebatch{\lrVert{\hat{\G}^t - \G^{t_0}}^2}  
    \nonumber \\ &~~~~~~~~~~~~~~~~~~~~~~~~~~~~~~~~~~~~~~~~~~~
    + LQ \sum_{t=t_0}^{t_0^+}(\eta^{t_0})^2\Ebatch{\lrVert{\G^{t_0}}^2} \label{exp1.eq} \\
    &\leq -\frac{1}{2}\sum_{t=t_0}^{t_0^+} \eta^{t_0}(1-2LQ\eta^{t_0}) \lrVert{\nabla F(\Theta^{t_0})}^2 
    \nonumber \\ &~~~
    + \frac{1}{2}\sum_{t=t_0}^{t_0^+} \eta^{t_0}(1+2LQ\eta^{t_0}) \Ebatch{\lrVert{\hat{\G}^t - \G^{t_0}}^2}  
    + LQ \sum_{m=0}^M \frac{\sigma_m^2}{B} \sum_{t=t_0}^{t_0^+}(\eta^{t_0})^2 \\
    &= -\frac{Q}{2}\eta^{t_0}(1-2LQ\eta^{t_0}) \lrVert{\nabla F(\Theta^{t_0})}^2 
    \nonumber \\ &~~~
    + \frac{1}{2}\sum_{t=t_0}^{t_0^+} \eta^{t_0}(1+2LQ\eta^{t_0}) \Ebatch{\lrVert{\hat{\G}^t - \G^{t_0}}^2}  
    + LQ^2 (\eta^{t_0})^2 \sum_{m=0}^M \frac{\sigma_m^2}{B} \label{var1.eq}
\end{align}
where (\ref{exp1.eq}) follows from applying Assumption~\ref{bias.assum}
and noting that $\Ebatch{\G^{t_0}} = \nabla F(\Theta^{t_0})$,
and (\ref{var1.eq}) follows from Assumption~\ref{var.assum}.

Applying Lemma~\ref{diff.lemma} to (\ref{var1.eq}):
\begin{align}
    &\Ebatch{F(\Theta^{t_0^+})} - F(\Theta^{t_0})
    \leq -\frac{Q}{2}\eta^{t_0}(1-2LQ\eta^{t_0}) \lrVert{\nabla F(\Theta^{t_0})}^2 
    \nonumber \\ &~~~~~~~~~~~~~~~~~~~~~~~~~~~~~~~~~~~~~~~~~~~~~~~~~~
    + 8Q^3(\eta^{t_0})^3 (1+2LQ\eta^{t_0})
    \sum_{m=0}^M L_m^2\lrVert{\nabla_m F(\Theta_m^{t_0})}^2    
    \nonumber \\ &~~~~~~~~~~~~~~~~~~~~~~~~~~~~~~~~~~~~~~~~~~~~~~~~~~
    + 8Q^3(\eta^{t_0})^3(1+2LQ\eta^{t_0})   
    \sum_{m=0}^M L_m^2 \frac{\sigma_m^2}{B} 
    \nonumber \\ &~~~~~~~~~~~~~~~~~~~~~~~~~~~~~~~~~~~~~~~~~~~~~~~~~~
    + 32Q^3\eta^{t_0}(1+2LQ\eta^{t_0})   
    \sum_{m=0}^M H_m^2 G_m^2 \lrVert{E_m^{t_0}}_{\mathcal{F}}^2 
    \nonumber \\ &~~~~~~~~~~~~~~~~~~~~~~~~~~~~~~~~~~~~~~~~~~~~~~~~~~
    + LQ^2 (\eta^{t_0})^2 \sum_{m=0}^M \frac{\sigma_m^2}{B}\\
    &~~~~~\leq 
    -\frac{Q}{2}\sum_{m=0}^M \eta^{t_0}(1-2LQ\eta^{t_0} - 16Q^2L_m^2(\eta^{t_0})^2  - 16Q^3L_m^2L(\eta^{t_0})^3)) \lrVert{\nabla_m F(\Theta^{t_0})}^2 
    \nonumber \\ &~~~~~~~~~~~~~~~~~~
    + (LQ^2(\eta^{t_0})^2+8Q^3L_m^2(\eta^{t_0})^3+8Q^4LL_m^2(\eta^{t_0})^4) 
    \sum_{m=0}^M \frac{\sigma_m^2}{B} 
    \nonumber \\ &~~~~~~~~~~~~~~~~~~
    + 32Q^3 \eta^{t_0}(1+2LQ\eta^{t_0})   
    \sum_{m=0}^M H_m^2 G_m^2 \lrVert{E_m^{t_0}}_{\mathcal{F}}^2. \label{needeta2.eq}
\end{align}

Let $\eta^{t_0} \leq \frac{1}{16Q\max\{L,\max_m L_m\}}$. Then we bound (\ref{needeta2.eq}) further:
\begin{align}
    \Ebatch{F(\Theta^{t_0^+})} - F(\Theta^{t_0})
    &\leq -\frac{Q}{2}\sum_{m=0}^M \eta^{t_0}\left(1-\frac{1}{8} - \frac{1}{16}  - \frac{1}{16^2}\right) \lrVert{\nabla_m F(\Theta^{t_0})}^2 
    \nonumber \\ &~~~
    + (LQ^2(\eta^{t_0})^2+8Q^3L_m^2(\eta^{t_0})^3+8Q^4LL_m^2(\eta^{t_0})^4) 
    \sum_{m=0}^M \frac{\sigma_m^2}{B} 
    \nonumber \\ &~~~
    + 16Q^3 \eta^{t_0}(1+2LQ\eta^{t_0})   
    \sum_{m=0}^M H_m^2 G_m^2 \lrVert{E_m^{t_0}}_{\mathcal{F}}^2 \\ 
    &\leq -\frac{3Q}{8}\eta^{t_0} \lrVert{\nabla F(\Theta^{t_0})}^2 
    \nonumber \\ &~~~
    + (LQ^2(\eta^{t_0})^2+8Q^3L_m^2(\eta^{t_0})^3+8Q^4LL_m^2(\eta^{t_0})^4) 
    \sum_{m=0}^M \frac{\sigma_m^2}{B} 
    \nonumber \\ &~~~
    + 32Q^3 \eta^{t_0}(1+2LQ\eta^{t_0})   
    \sum_{m=0}^M H_m^2 G_m^2 \lrVert{E_m^{t_0}}_{\mathcal{F}}^2  
\end{align}

After some rearranging of terms:
\begin{align}
    \eta^{t_0} \lrVert{\nabla F(\Theta^{t_0})}^2 
    &\leq \frac{4\left[F(\Theta^{t_0}) - \Ebatch{F(\Theta^{t_0^+})}\right]}{Q}
    \nonumber \\ &~~~
    + \frac{8}{3}(LQ(\eta^{t_0})^2+8Q^2L_m^2(\eta^{t_0})^3+8Q^3LL_m^2(\eta^{t_0})^4) 
    \sum_{m=0}^M \frac{\sigma_m^2}{B} 
    \nonumber \\ &~~~
    + 86Q^2 \eta^{t_0}(1+2LQ\eta^{t_0})   
    \sum_{m=0}^M H_m^2 G_m^2 \lrVert{E_m^{t_0}}_{\mathcal{F}}^2 
\end{align}

Summing over all global rounds $t_0=0,\ldots,R-1$ and taking total expectation:
\begin{align}
    \sum_{t_0=0}^{R-1} \eta^{t_0} \Etot{\lrVert{\nabla F(\Theta^{t_0})}^2} 
    &\leq \frac{4\left[F(\Theta^0) - \Etot{F(\Theta^T)}\right]}{Q}
    \nonumber \\ &~~~
    + \frac{8}{3} \sum_{t_0=0}^{R-1}
    (LQ(\eta^{t_0})^2+8Q^2L_m^2(\eta^{t_0})^3+8Q^3LL_m^2(\eta^{t_0})^4) 
    \sum_{m=0}^M \frac{\sigma_m^2}{B} 
    \nonumber \\ &
    + 86Q^2 \eta^{t_0}(1+2LQ\eta^{t_0})   
    \sum_{m=0}^M H_m^2 G_m^2 \lrVert{E_m^{t_0}}_{\mathcal{F}}^2 \\ 
    &\leq \frac{4\left[F(\Theta^0) - \Etot{F(\Theta^T)}\right]}{QR}
    \nonumber \\ &~~~
    + \frac{8}{3}\sum_{t_0=0}^{R-1}(QL(\eta^{t_0})^2+8Q^2L_m^2(\eta^{t_0})^3+8Q^3LL_m^2(\eta^{t_0})^4) \sum_{m=0}^M \frac{\sigma_m^2}{B} 
    \nonumber \\ &
    + 86Q^2 \sum_{t_0=0}^{R-1}\eta^{t_0}(1+2LQ\eta^{t_0})   
    \sum_{m=0}^M H_m^2 G_m^2 \Etot{\lrVert{E_m^{t_0}}_{\mathcal{F}}^2}.
     \label{before_theoremsplit.eq}
\end{align}

Note that:
\begin{align}
    \sum_{m=0}^MH_m^2 G_m^2 \Etot{\lrVert{E_m^{t_0}}_{\mathcal{F}}^2} 
    &=\sum_{m=0}^M H_m^2 G_m^2 \sum_{j \neq m} \Etot{\lrVert{\epsilon_j^{t_0}}_{\mathcal{F}}^2} \\
    &= \sum_{m=0}^M H_m^2 G_m^2 \sum_{j \neq m} \mathcal{E}_j^{t_0} \label{error_batch.eq} 
\end{align}
where (\ref{error_batch.eq}) follows from Definition~\ref{compress.assum}.  

Plugging this into (\ref{before_theoremsplit.eq})
\begin{align}
    \sum_{t_0=0}^{R-1} \eta^{t_0} \Etot{\lrVert{\nabla F(\Theta^{t_0})}^2} 
    &\leq \frac{4\left[F(\Theta^0) - \Etot{F(\Theta^T)}\right]}{QR}
    \nonumber \\ &~~~
    + \frac{8}{3}\sum_{t_0=0}^{R-1}(QL(\eta^{t_0})^2+8Q^2L_m^2(\eta^{t_0})^3+8Q^3LL_m^2(\eta^{t_0})^4) \sum_{m=0}^M \frac{\sigma_m^2}{B} 
    \nonumber \\ &
    + 86Q^2 \sum_{t_0=0}^{R-1}\eta^{t_0}(1+2LQ\eta^{t_0})   
    \sum_{m=0}^M H_m^2 G_m^2 \sum_{j \neq m}\mathcal{E}_j^{t_0}.
     \label{theoremsplit.eq}
\end{align}

Suppose that $\eta^{t_0}=\eta$ for all global rounds $t_0$.
Then, averaging over $R$ global rounds, we have:
\begin{align}
    \frac{1}{R} \sum_{t_0=0}^{R-1} \Etot{\lrVert{\nabla F(\Theta^{t_0})}^2} 
    &\leq \frac{4\left[F(\Theta^0) - \Etot{F(\Theta^T)}\right]}{QR\eta}
    + \frac{8}{3}\sum_{m=0}^M (QL\eta+8Q^2L_m^2\eta^2+8Q^3LL_m^2\eta^3) \frac{\sigma_m^2}{B} 
    \nonumber \\ &~~~
    + \frac{86Q^2}{R} \sum_{t_0=0}^{R-1}(1+2LQ\eta)   
    \sum_{m=0}^M H_m^2 G_m^2 \sum_{j \neq m}\mathcal{E}_j^{t_0}.\\
    &\leq \frac{4\left[F(\Theta^0) - \Etot{F(\Theta^T)}\right]}{QR\eta}
    + 6\sum_{m=0}^M QL\eta \frac{\sigma_m^2}{B} 
    + \frac{92Q^2}{R} \sum_{t_0=0}^{R-1}   
    \sum_{m=0}^M H_m^2 G_m^2 \sum_{j \neq m}\mathcal{E}_j^{t_0}.
    \label{thm1.eq}
\end{align}
where (\ref{thm1.eq}) follows from our assumption that 
$\eta^{t_0} \leq \frac{1}{16Q\max\{L,\max_m L_m\}}$.
This completes the proof of Theorem \ref{main.thm}.

We continue our analysis to prove Theorem \ref{main2.thm}.
Starting from (\ref{theoremsplit.eq}),
we bound the left-hand side with the minimum over all iterations:
\begin{align}
    &\min_{t_0=0,\ldots,R-1} \Etot{\lrVert{\nabla F(\Theta^{t_0})}^2}
    \leq \frac{4\left[F(\Theta^0) - \Ebatch{F(\Theta^T)}\right]}{Q\sum_{t_0=0}^{R-1} \eta^{t_0}}
    \nonumber \\ &~~~
    + 4\left(QL\frac{\sum_{t_0=0}^{R-1}(\eta^{t_0})^2}{\sum_{t_0=0}^{R-1}\eta^{t_0}}
    +16Q^2L_m^2\frac{\sum_{t_0=0}^{R-1}(\eta^{t_0})^3}{\sum_{t_0=0}^{R-1}\eta^{t_0}}
    +16Q^3LL_m^2\frac{\sum_{t_0=0}^{R-1}(\eta^{t_0})^4}{\sum_{t_0=0}^{R-1}\eta^{t_0}}\right) 
    \sum_{m=0}^M \frac{\sigma_m^2}{B} 
    \nonumber \\ &~~~
    + 86Q^2\sum_{m=0}^M H_m^2 G_m^2  \frac{\sum_{t_0=0}^{R-1}\eta^{t_0}
    \sum_{j \neq m}\mathcal{E}_j^{t_0}}{\sum_{t_0=0}^{R-1}\eta^{t_0}}
    + 172LQ^3\sum_{m=0}^M H_m^2 G_m^2  \frac{\sum_{t_0=0}^{R-1}(\eta^{t_0})^2
    \sum_{j \neq m}\mathcal{E}_j^{t_0}}{\sum_{t_0=0}^{R-1}\eta^{t_0}}
\end{align}
      
As $R \rightarrow \infty$, if $\sum_{t_0=0}^{R-1}\eta^{t_0} = \infty$,
$\sum_{t_0=0}^{R-1}(\eta^{t_0})^2 < \infty$,
and $\sum_{t_0=0}^{R-1}\eta^{t_0}\sum_{j \neq m}\mathcal{E}_j^{t_0} < \infty$,
then $\min_{t_0=0,\ldots,R-1} \Etot{\lrVert{\nabla F(\Theta^{t_0})}^2} \rightarrow 0$.
This completes the proof of Theorem~\ref{main2.thm}.

\section{Common Compressors} \label{common.sec}
In this section, we calculate the compression error and
parameter bounds for
uniform scalar quantization, lattice vector quantization
and top-$k$ sparsification, as well as discuss
implementation details of these compressors in \mbox{C-VFL}.

We first consider a uniform scalar quantizer~\citep{DBLP:journals/bstj/Bennett48}
with a set of $2^q$ quantization levels, 
where $q$ is the number of bits to represent compressed values.
We define the range of values that quantize to the same
quantization level as the quantization bin.
In \mbox{C-VFL}, a scalar quantizer  
quantizes each individual component of embeddings.
The error in each embedding of a batch $\B$ 
in scalar quantization is 
$\leq P_m \frac{\Delta^2}{12} = P_m\frac{(h_{max} - h_{min})^2}{12}2^{-2q}$
where $\Delta$ the size of a quantization bin,
$P_m$ is the size of the $m$-th embedding,
$h_{max}$ and $h_{min}$ are respectively the
maximum and minimum value $h_m(\theta_m^t; x_m^i)$ 
can be for all iterations $t$, parties $m$, and $x_m^i$.
We note that if $h_{max}$ or $h_{min}$ are unbounded, then
the error is unbounded as well.
By Theorem~\ref{main.thm}, we know that 
$\frac{1}{R}\sum_{t_0=0}^{R-1} \sum_{m=0}^M \mathcal{E}_m^{t_0} = O(\frac{1}{\sqrt{T}})$
to obtain a convergence rate of $O(\frac{1}{\sqrt{T}})$.
If we use the same $q$ for all parties and iterations,
we can solve for $q$ to find that the value $q$ must be lower bounded by 
$q = \Omega (\log_2 (P_m(h_{max} - h_{min})^2\sqrt{T}))$
to reach a convergence rate of $O(\frac{1}{\sqrt{T}})$.
For a diminishing compression error, required by Theorem~\ref{main2.thm},
we let $T=t_0$ in this bound, indicating that $q$, the number of quantization
bins, must increase as training continues.

A vector quantizer creates a set of $d$-dimensional vectors
called a codebook~\citep{DBLP:journals/tit/ZamirF96}.
A vector is quantized by dividing the components into sub-vectors
of size $d$,
then quantizing each sub-vector to the nearest codebook vector in Euclidean distance.
A cell in vector quantization is defined as all points in $d$-space that
quantizes to a single codeword. The volume of these cells
are determined by how closely packed codewords are.
We consider the commonly applied $2$-dimensional hexagonal lattice quantizer~\citep{DBLP:journals/tsp/ShlezingerCEPC21}.
In \mbox{C-VFL}, each embedding is divided into sub-vectors
of size two, scaled to the unit square,
then quantized to the nearest vector by Euclidean distance in the codebook.
The error in this vector quantizer is $\leq \frac{VP_m}{24}$
where $V$ is the volume of a lattice cell.
The more bits available for quantization,
the smaller the volume of the cells,
the smaller the compression error.
We can calculate an upper bound on $V$ based on Theorem~\ref{main.thm}:
$V = O(\frac{1}{P_m \sqrt{T}})$.
If a diminishing compression error is required,
we can set $T=t_0$ in this bound, indicating that $V$ must decrease
at a rate of $O(\frac{1}{P_m \sqrt{t_0}})$.
As the number of iterations increases, the smaller $V$ must be,
and thus the more bits that must be communicated.

In top-$k$ sparsification~\citep{DBLP:conf/iclr/LinHM0D18}, 
when used in distributed SGD algorithms,
the $k$ largest magnitude components of the gradient are sent while
the rest are set to zero.
In the case of embeddings in \mbox{C-VFL}, 
a large element may be as important 
as an input to the server model as a small one.
We can instead select the $k$ embedding
elements to send with the largest magnitude partial derivatives in
$\nabla_{\theta_m} h_m(\theta_m^t)$.
Since a party $m$ cannot calculate $\nabla_{\theta_m} h_m(\theta_m^t)$
until all parties send their embeddings, 
party $m$ can use the embedding gradient 
calculated in the previous iteration, 
$\nabla_{\theta_m} h_m(\theta_m^{t-1})$. 
This is an intuitive method, as we assume our gradients are
Lipschitz continuous, and thus do not change too rapidly.
The error of sparsification is
$\leq (1-\frac{k}{P_m})(\|h\|^2)_{max}$
where $(\|h\|^2)_{max}$ is the maximum value
of $\|h_m(\theta_m^t; x_m^i)\|^2$ 
for all iterations $t$, parties $m$, and $x_m^i$. 
Note that if $(\|h\|^2)_{max}$ is unbounded, then the error is unbounded.
We can calculate a lower bound on $k$:
$k = \Omega (P_m - \frac{P_m}{(\|h\|^2)_{max}\sqrt{T}})$.
Note that the larger $(\|h\|^2)_{max}$, the larger $k$ must be. 
More components must be sent if embedding magnitude is large
in order to achieve a convergence rate of $O(\frac{1}{\sqrt{T}})$.
When considering a diminishing compression error,
we set $T=t_0$, showing that $k$ must increase
over the course of training.

\section{Experimental Details} \label{exp_details.sec}
For our experiments, we used an internal cluster 
of $40$ compute nodes running CentOS 7 each with 
$2\times$ $20$-core $2.5$ GHz Intel Xeon Gold $6248$ CPUs,
$8\times$ NVIDIA Tesla V100 GPUs with $32$ GB HBM,
and $768$ GB of RAM.

\subsection{MIMIC-III}
The MIMIC-III dataset can be found at: \href{https://mimic.physionet.org/}{mimic.physionet.org}.
The dataset consists of time-series data from ${\sim}60$,$000$ intensive
care unit admissions. The data includes many features about each patient,
such as demographic, vital signs, medications, and more. All
the data is anonymized.
In order to gain access to the dataset, one must take the short 
online course provided on their website.

Our code for training with the MIMIC-III dataset can be found in
in the folder titled ``mimic3". This is an extension of
the MIMIC-III benchmarks repo found at: 
\href{https://github.com/YerevaNN/mimic3-benchmarks}{github.com/YerevaNN/mimic3-benchmarks}. 
The original code preprocesses the MIMIC-III dataset and provides
starter code for training LSTMs using centralized SGD. Our code
has updated their existing code to TensorFlow 2.
The new file of interest in our code base
is ``mimic3models/in\_hospital\_mortality/quant.py" 
which runs \mbox{C-VFL}. 
Both our code base and the original are under the MIT License.
More details on installation, dependencies, and running 
our experiments can be found in ``README.md". 
Each experiment took approximately six hours to run on a node in our cluster.

The benchmarking preprocessing code splits the data up into
different prediction cases. Our experiments train models
to predict for in-hospital mortality.
For in-hospital mortality, there are $14$,$681$ training samples,
and $3$,$236$ test samples.  
In our experiments, we use a step size of $0.01$, as is standard
for training an LSTM on the MIMIC-III dataset.

\subsection{ModelNet10 and CIFAR10}
Details on the ModelNet10 dataset can be found at: \href{https://modelnet.cs.princeton.edu/}{modelnet.cs.princeton.edu/}.
The specific link we downloaded the dataset 
from is the following Google Drive link: \href{https://drive.google.com/file/d/0B4v2jR3WsindMUE3N2xiLVpyLW8/view}{https://drive.google.com/file/d/0B4v2jR3WsindMUE3N2xiLVpyLW8/view}.
The dataset consists of 3D CAD models of different common
objects in the world. For each CAD model, there are
12 views from different angles saved as PNG files. 
We only trained our models on the following 10 classes:
bathtub,  bed,  chair,  desk,  dresser,  monitor,  night\_stand,  sofa,  table,  toilet.
We used a subset of the data with
$1$,$008$ training samples and $918$ test samples.
In our experiments, we use a step size of $0.001$, as is standard
for training a CNN on the ModelNet10 dataset.

Our code for learning on the ModelNet10 dataset is in the folder
``MVCNN\_Pytorch" and is an extension of the MVCNN-PyTorch
repo: \href{https://github.com/RBirkeland/MVCNN-PyTorch}{github.com/RBirkeland/MVCNN-PyTorch}.
The file of interest in our code base
is ``quant.py" which runs \mbox{C-VFL}. 
Both our code base and the original are under the MIT License.
Details on how to run our experiments can be found in the ``README.md".
Each experiment took approximately six hours to run on a node in our cluster.

In the same folder, ``MVCNN\_Pytorch", we include our code for running CIFAR-10.
The file of interest is ``quant\_cifar.py" which trains \mbox{C-VFL} with CIFAR-10.
We use the version of CIFAR-10 downloaded through the torchvision library.
More information on the CIFAR-10 dataset can be found at: 
\href{https://www.cs.toronto.edu/~kriz/cifar.html}{cs.toronto.edu/~kriz/cifar.html}.

\subsection{ImageNet}
In Section~\ref{additional.sec}, we include additional experiments 
that use the ImageNet dataset.
Details on ImageNet can be found at: \href{https://www.image-net.org/}{image-net.org/}.
We specifically use a random 100-class subset from the 2012 ILSVRC version of the data.

Our code for learning on the ImageNet dataset is in the folder
``ImageNet\_CVFL" and is a modification on the moco\_align\_uniform repo: 
\href{https://github.com/SsnL/moco_align_uniform}{https://github.com/SsnL/moco\_align\_uniform}.
The file of interest in our code base
is ``main\_cvfl.py" which runs \mbox{C-VFL}. 
Both our code base and the original are under the CC-BY-NC 4.0 license.
Details on how to run our experiments can be found in the ``README.txt".
Each experiment took approximately $24$ hours to run on a node in our cluster.

\section{Additional Plots and Experiments} \label{additional.sec}

In this section, we include additional plots using
the results from the experiments introduced in 
Section~\ref{exp.sec} of the main paper. 
We also include new experiments with the ImageNet100 dataset.
Finally, we include additional experiments with an alternate
\mbox{C-VFL} for $Q=1$.

\subsection{Additional Plots}
We first provide additional plots from the experiements in the main paper.
The setup for the experiments is described in the main paper.
These plots provide some additional insight into the
effect of each compressor on convergence in all datasets.
As with the plots in the main paper, 
the solid lines in each plot are the average of five runs 
and the shaded regions represent one standard deviation.

\begin{figure}[H]
    \centering
    \begin{subfigure}{0.275\textwidth}
        \centering
        \includegraphics[width=\textwidth]{images/accs_test_quant_4.png}
        \caption{$2$ bits per parameter}
        \label{bit2mimic.fig}
    \end{subfigure}
    \begin{subfigure}{0.275\textwidth}
        \centering
        \includegraphics[width=\textwidth]{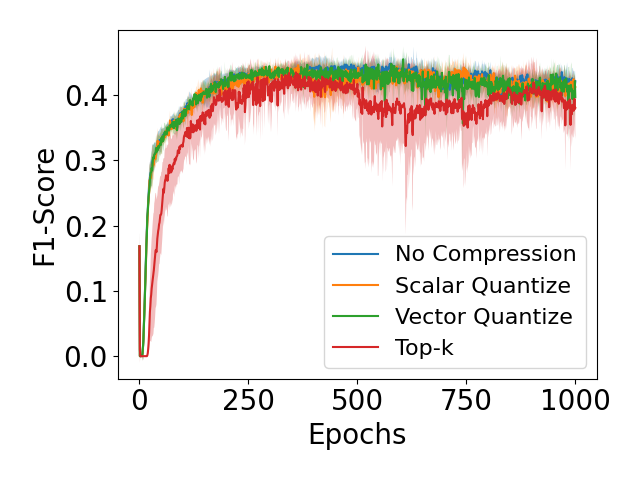}
        \caption{$3$ bits per parameter}
        \label{bit3mimic.fig}
    \end{subfigure}
    \begin{subfigure}{0.275\textwidth}
        \centering
        \includegraphics[width=\textwidth]{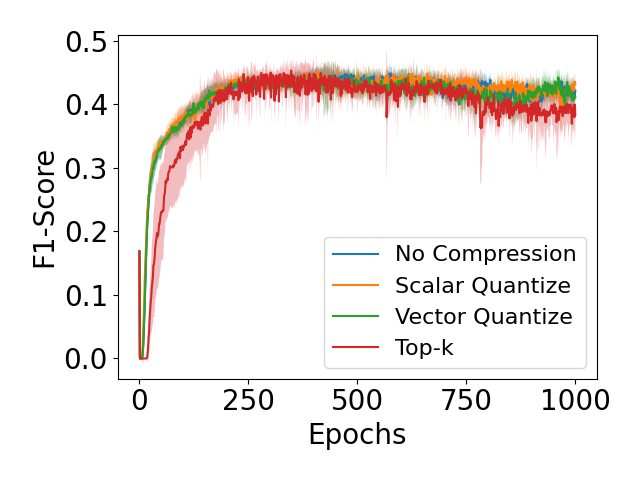}
        \caption{$4$ bits per parameter}
        \label{bit4mimic.fig}
    \end{subfigure}
    \caption{Test $F_1$-Score on MIMIC-III dataset. Scalar and vector
            In these experiments, $Q=10$ and $M=4$.
        quantization achieve similar test $F_1$-Score even when only
        using $2$ bits in quantization. 
        On the other hand, top-$k$ sparsification performs worse than the
        other compressors in the MIMIC-III dataset.}
    \label{mimic_acc.fig}
\end{figure}

Figure~\ref{mimic_acc.fig} plots the test $F_1$-Score for
training on the MIMIC-III dataset for different levels of compression. 
We can see that scalar and vector quantization perform
similarly to no compression and improve as the number
of bits available increase.
We can also see that top-$k$ sparsification has high
variability on the MIMIC-III dataset and generally performs
worse than the other compressors. 

\begin{figure}[H]
    \centering
    \begin{subfigure}{0.275\textwidth}
        \centering
        \includegraphics[width=\textwidth]{images/accs_test_quantcomm_4.png}
        \caption{$2$ bits per parameter}
        \label{bit2mimic.fig}
    \end{subfigure}
    \begin{subfigure}{0.275\textwidth}
        \centering
        \includegraphics[width=\textwidth]{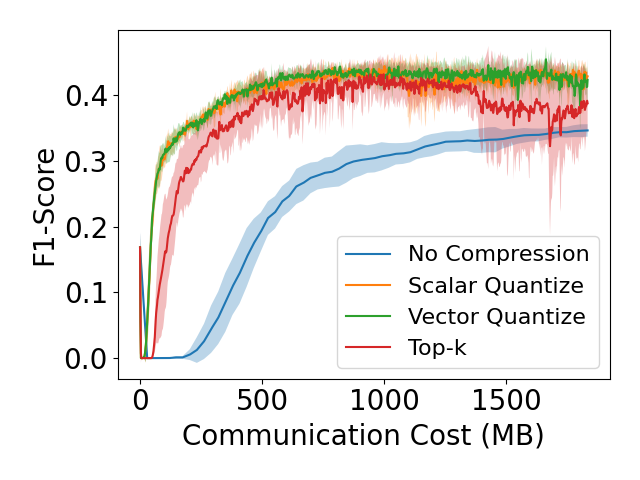}
        \caption{$3$ bits per parameter}
        \label{bit3mimic.fig}
    \end{subfigure}
    \begin{subfigure}{0.275\textwidth}
        \centering
        \includegraphics[width=\textwidth]{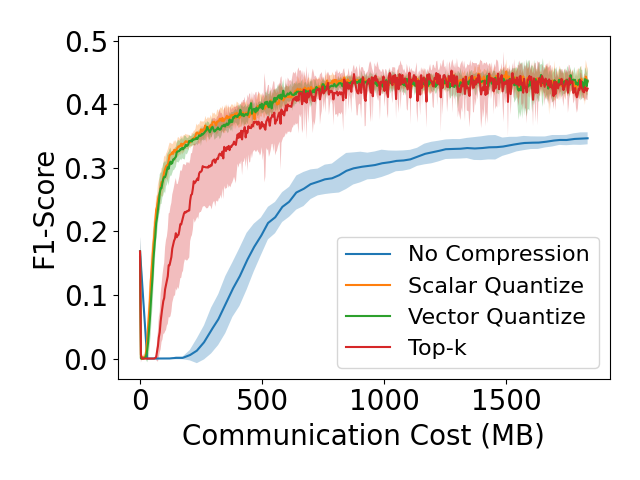}
        \caption{$4$ bits per parameter}
        \label{bit4mimic.fig}
    \end{subfigure}
    \caption{Test $F_1$-Score on MIMIC-III dataset plotted by communication cost.
            In these experiments, $Q=10$ and $M=4$.
        We can see that all compressors reach higher $F_1$-Scores with 
        lower communication cost than no compression. We can see that the standard
        deviation for each compressor decreases as the number of bits available 
        increases. Top-$k$ sparsification generally performs worse than the 
        other compressors on the MIMIC-III-dataset.}
    \label{mimiccomm_acc.fig}
\end{figure}

\begin{figure}[H]
    \centering
    \begin{subfigure}{0.275\textwidth}
        \centering
        \includegraphics[width=\textwidth]{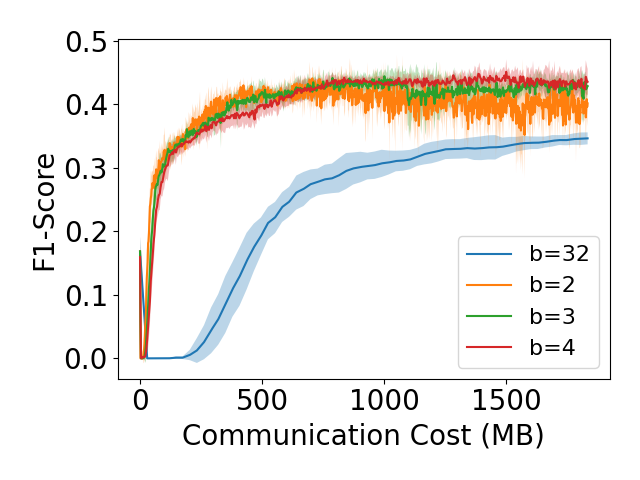}
        \caption{Scalar quantization}
        \label{bit2mimic.fig}
    \end{subfigure}
    \begin{subfigure}{0.275\textwidth}
        \centering
        \includegraphics[width=\textwidth]{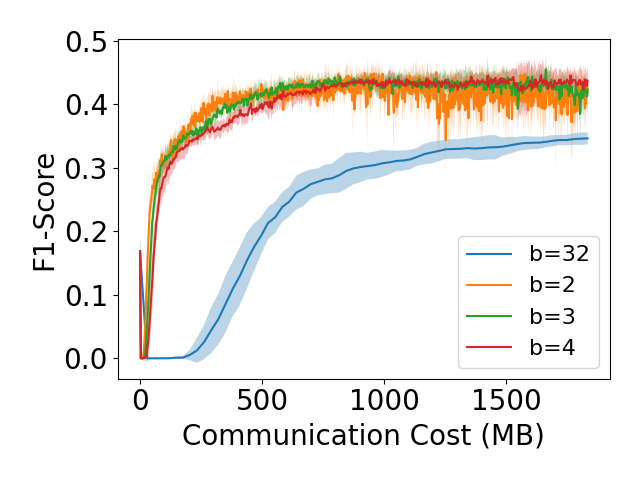}
        \caption{Vector quantization}
        \label{bit3mimic.fig}
    \end{subfigure}
    \begin{subfigure}{0.275\textwidth}
        \centering
        \includegraphics[width=\textwidth]{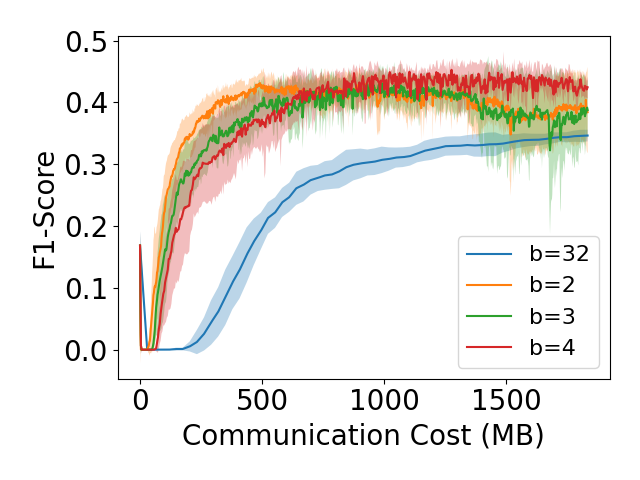}
        \caption{Top-$k$ sparsification}
        \label{bit4mimic.fig}
    \end{subfigure}
    \caption{Test $F_1$-Score on MIMIC-III dataset plotted by communication cost.
            In these experiments, $Q=10$ and $M=4$.
        We can see that all compressors reach higher $F_1$-Scores with 
        lower communication cost than no compression. We can see that the variability 
        for each compressor decreases as the number of bits available 
        increases.}
    \label{mimicquant_acc.fig}
\end{figure}

Figures~\ref{mimiccomm_acc.fig} and \ref{mimicquant_acc.fig}
plot the test $F_1$-Score for training on the 
MIMIC-III dataset plotted against the communication cost.
The plots in Figure~\ref{mimiccomm_acc.fig} include
all compression techniques for a given level of compression, 
while the plots in Figure~\ref{mimicquant_acc.fig} 
include all levels of compression for a given compression technique.
We can see that all compressors reach higher $F_1$-Scores
with lower communication cost than no compression.
It is interesting to note that increasing the number of bits per
parameter reduces the variability in all compressors.

    \begin{figure}[H]
    \centering
        \begin{subfigure}{0.275\textwidth}
            \centering
            \includegraphics[width=\textwidth]{images/accs_test_quant_mvcnn4.png}
            \caption{$2$ bits per parameter}
            \label{bit2mvcnn.fig}
        \end{subfigure}
        \begin{subfigure}{0.275\textwidth}
            \centering
            \includegraphics[width=\textwidth]{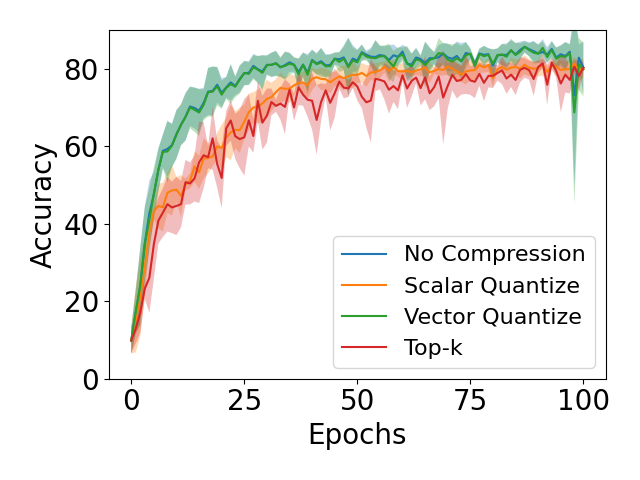}
            \caption{$3$ bits per parameter}
            \label{bit3mvcnn.fig}
        \end{subfigure}
        \begin{subfigure}{0.275\textwidth}
            \centering
            \includegraphics[width=\textwidth]{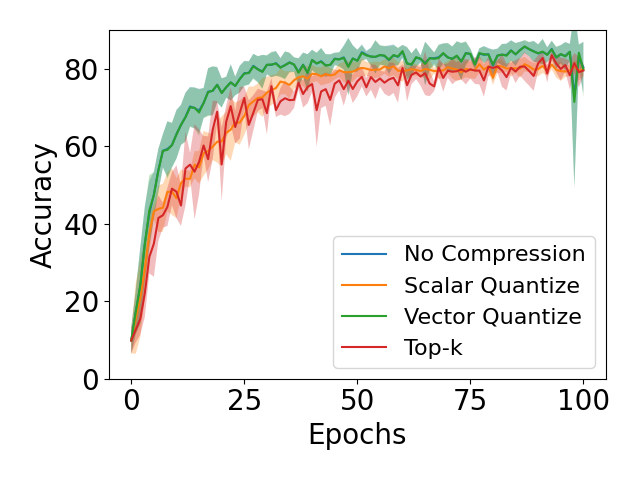}
            \caption{$4$ bits per parameter}
            \label{bit4mvcnn.fig}
        \end{subfigure}
        \caption{Test accuracy on ModelNet10 dataset. Vector quantization
            In these experiments, $Q=10$ and $M=4$.
            and top-$k$ sparsification perform similarly to no compression, 
            even when only $2$ bits are available. Scalar quantization
            converges to a lower test accuracy, 
            and has high variability on the ModelNet10 dataset.}
        \label{mvcnn_acc.fig}
    \end{figure}

    Figure~\ref{mvcnn_acc.fig}
    plots the test accuracy for
    training on the ModelNet10 dataset. 
    Vector quantization and top-$k$ sparsification 
    perform similarly to no compression, 
    even when only $2$ bits are available.
    We can see that scalar quantization has high
    variability on the ModelNet10 dataset.

    \begin{figure}[H]
            \centering
        \begin{subfigure}{0.275\textwidth}
            \centering
            \includegraphics[width=\textwidth]{images/accs_test_quantcomm_4_mvcnn.png}
            \caption{$2$ bits per parameter}
            \label{bit2mimic.fig}
        \end{subfigure}
        \begin{subfigure}{0.275\textwidth}
            \centering
            \includegraphics[width=\textwidth]{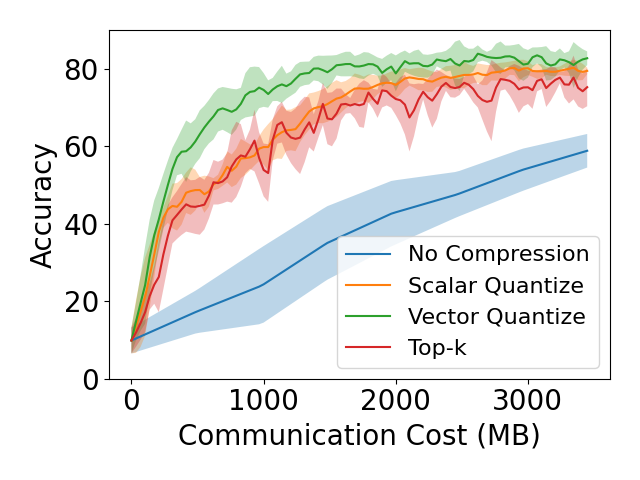}
            \caption{$3$ bits per parameter}
            \label{bit3mimic.fig}
        \end{subfigure}
        \begin{subfigure}{0.275\textwidth}
            \centering
            \includegraphics[width=\textwidth]{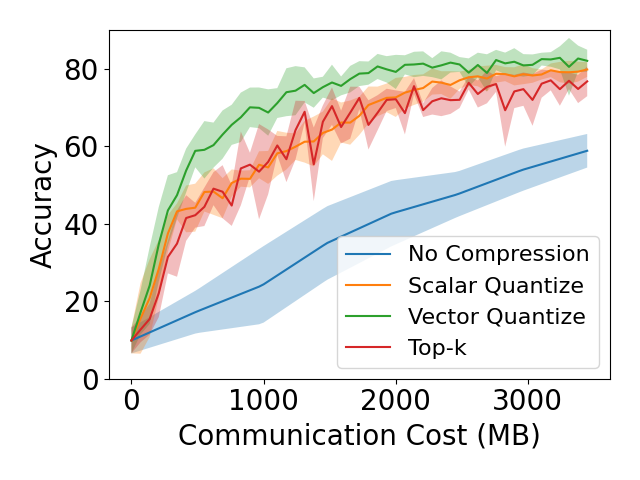}
            \caption{$4$ bits per parameter}
            \label{bit4mimic.fig}
        \end{subfigure}
        \caption{Test accuracy on ModelNet10 dataset plotted by communication cost.
            In these experiments, $Q=10$ and $M=4$.
            We can see that all compressors reach higher accuracies with 
            lower communication cost than no compression. 
            Scalar quantization generally performs worse than the 
            other compressors on the ModelNet10 dataset.}
        \label{mvcnncomm_acc.fig}
    \end{figure}

    \begin{figure}[H]
            \centering
        \begin{subfigure}{0.275\textwidth}
            \centering
            \includegraphics[width=\textwidth]{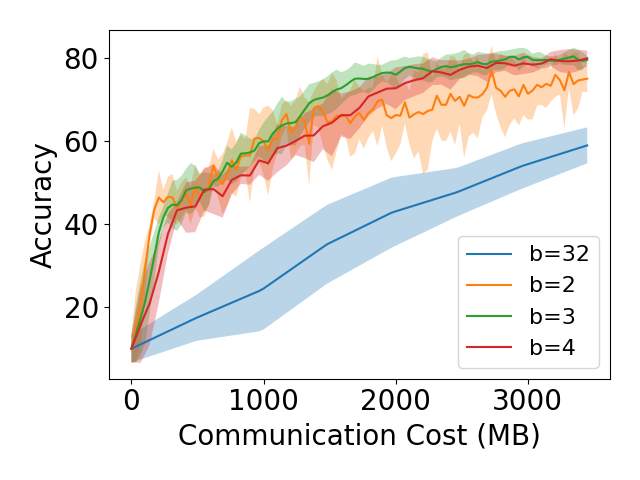}
            \caption{Scalar quantization}
            \label{bitscalarmvcnn.fig}
        \end{subfigure}
        \begin{subfigure}{0.275\textwidth}
            \centering
            \includegraphics[width=\textwidth]{images/accs_test_quantcomm_quantize2_mvcnn.png}
            \caption{Vector quantization}
            \label{bitvectormvcnn.fig}
        \end{subfigure}
        \begin{subfigure}{0.275\textwidth}
            \centering
            \includegraphics[width=\textwidth]{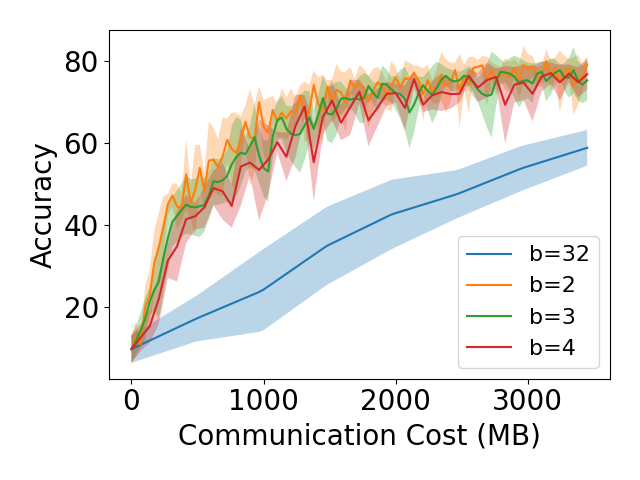}
            \caption{Top-$k$ sparsification}
            \label{bittopkmvcnn.fig}
        \end{subfigure}
        \caption{Test accuracy on ModelNet10 dataset plotted by communication cost.
            In these experiments, $Q=10$ and $M=4$.
            We can see that all compressors reach higher accuracies with 
            lower communication cost than no compression. 
            We can see that when less bits are used in each compressor,
            higher test accuracies are reached at lower communication costs.
            Scalar quantization generally performs worse than the 
            other compressors on the ModelNet10 dataset.}
        \label{mvcnnquant_acc.fig}
    \end{figure}

Figures~\ref{mvcnncomm_acc.fig} and \ref{mvcnnquant_acc.fig}
plot the test accuracy for training on the 
ModelNet10 dataset against the communication cost. 
The plots in Figure~\ref{mvcnncomm_acc.fig} include
all compression techniques for a given level of compression, 
while the plots in Figure~\ref{mvcnnquant_acc.fig} 
include all levels of compression for a given compression technique.
We can see that all compressors reach higher accuracies with 
lower communication cost than no compression. 
Scalar quantization generally performs worse than the 
other compressors on the ModelNet10 dataset.
From Figure~\ref{mvcnnquant_acc.fig}, we also see
that when fewer bits are used in each compressor,
higher test accuracies are reached at lower communication costs.
    
\begin{figure}[H]
    \centering
    \begin{subfigure}{0.275\textwidth}
        \centering
        \includegraphics[width=\textwidth]{images/accs_test_quant_cifar4.png}
        \caption{Plotted by epochs}
        \label{cifar_epochs.fig}
    \end{subfigure}
    \begin{subfigure}{0.275\textwidth}
        \centering
        \includegraphics[width=\textwidth]{images/accs_test_quantcomm_4_cifar.png}
        \caption{Plotted by cost}
        \label{cifar_comm.fig}
    \end{subfigure}
    \begin{subfigure}{0.275\textwidth}
        \centering
        \includegraphics[width=\textwidth]{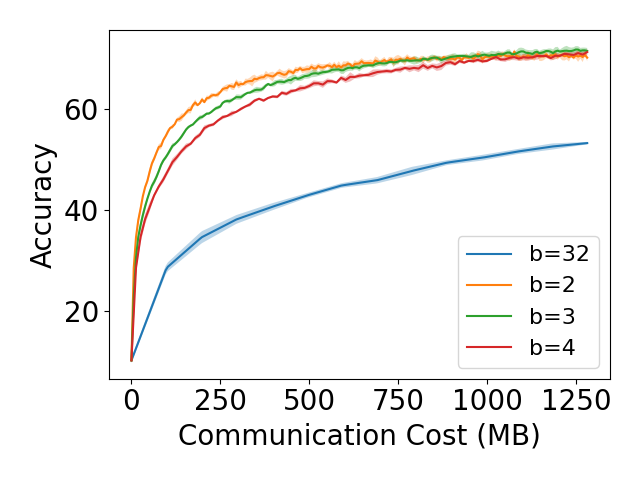}
        \caption{Vector quantization}
        \label{cifar_vector.fig}
    \end{subfigure}
    \caption{Test accuracy on CIFAR-10 dataset with the number of parties
    $M=4$ and number of local iterations $Q=10$. 
    In the first two plots, the compressors have $b=2$, 
    where $b$ is the number of bits used to represent embedding components.
    In the third plot, $b=32$ indicates there is no compression.
    The results show vector quantization performs the best our of the compressors,
    and all compressors show improvement over no compression in terms of
    communication cost to reach target test accuracies.}
    \label{cifar.fig}
\end{figure}

    In Figure~\ref{cifar.fig}, we plot the test accuracy for the CIFAR-10 dataset.
    The test accuracy is fairly low compared to typical baseline accuracies,
    which is expected, as learning object classification from only a quadrant of
    a $32 \times 32$ pixel image is difficult.
    Figure~\ref{cifar_epochs.fig} shows the test accuracy plotted by epochs.
    We can see that vector quantization performs almost as 
    well as no compression in the CIFAR-10 dataset.
    When plotting by communication cost, seen in Figure~\ref{cifar_comm.fig},
    we can see that vector quantization performs the best, though scalar
    quantization and top-$k$ sparsification show communication savings as well.
    In Figure~\ref{cifar_vector.fig}, we plot the test accuracy
    of \mbox{C-VFL} using vector quantization for different values of $b$, the
    number of bits to represent compressed values. Similar to previous
    results, lower $b$ tends to improve test accuracy reached with
    the same amount of communication cost.

    \subsection{Additional Experiments With ImageNet}

We also run \mbox{C-VFL} on ImageNet100~\cite{deng2009imagenet}.
ImageNet is a large image dataset for object classification. 
We use a random subset of 100 classes (ImageNet100) from the ImageNet dataset (about $126$,$000$ images).
We train a set of $4$ parties, each storing a different quadrant of every image. 
Each party trains ResNet18, and the server trains a fully-connected layer.
We use a variable step size, that starts at $0.001$, and drops to $0.0001$ after $50$ epochs.
We use a batch size of $256$ and train for $100$ epochs.
    
\begin{figure}[H]
        \centering
    \begin{subfigure}{0.275\textwidth}
        \centering
        \includegraphics[width=\textwidth]{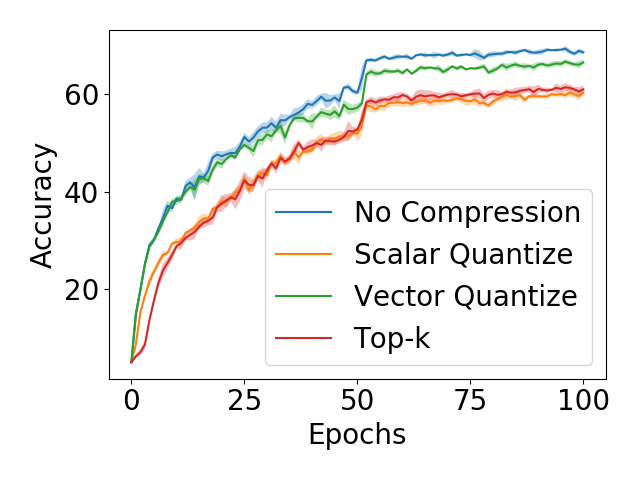}
        \caption{Plotted by epochs}
        \label{imagenet_epochs.fig}
    \end{subfigure}
    \begin{subfigure}{0.275\textwidth}
        \centering
        \includegraphics[width=\textwidth]{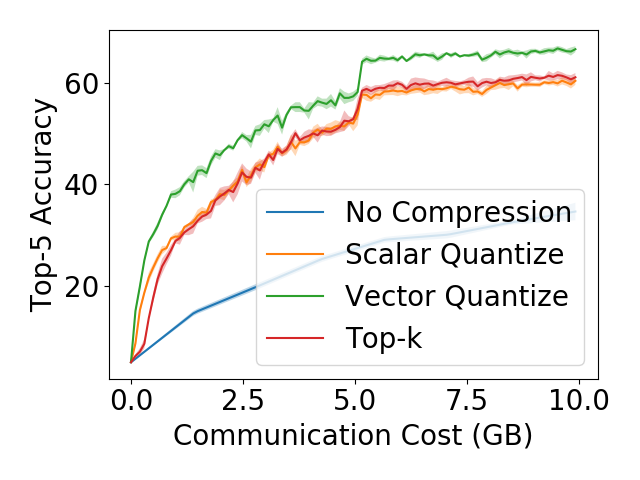}
        \caption{Plotted by cost}
        \label{imagenet_comm.fig}
    \end{subfigure}
    \begin{subfigure}{0.275\textwidth}
        \centering
        \includegraphics[width=\textwidth]{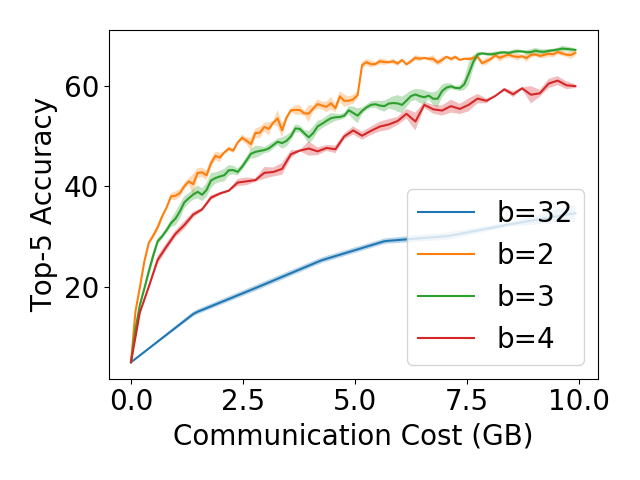}
        \caption{Vector quantization}
        \label{imagenet_vector.fig}
    \end{subfigure}
    \caption{Test accuracy on ImageNet-100 dataset with the number of parties
    $M=4$ and number of local iterations $Q=10$. 
    In the first two plots, the compressors have $b=2$, 
    where $b$ is the number of bits used to represent embedding components.
    In the third plot, $b=32$ indicates there is no compression.
    The results show vector quantization performs the best our of the compressors,
    and all compressors show improvement over no compression in terms of
    communication cost to reach target test accuracies.}
    \label{imagenet.fig}
\end{figure}

    In Figure~\ref{imagenet.fig}, we plot the top-$5$ test accuracy for ImageNet100.
    Figure~\ref{imagenet_epochs.fig} shows the test accuracy plotted by epochs.
    We can see that vector quantization performs almost as 
    well as no compression in the ImageNet100 dataset.
    When plotting by communication cost, seen in Figure~\ref{imagenet_comm.fig},
    we can see that vector quantization performs the best, though scalar
    quantization and top-$k$ sparsification show communication savings as well.
    In Figure~\ref{imagenet_vector.fig}, we plot the test accuracy
    of \mbox{C-VFL} using vector quantization for different values of $b$, the
    number of bits to represent compressed values. Similar to previous
    results, lower $b$ tends to improve test accuracy reached with
    the same amount of communication cost.

    \rev{
        \subsection{Comparison With Alternative \mbox{C-VFL} Algorithm For $Q=1$}
        In \mbox{C-VFL}, the server distributes party embeddings to all parties along with the 
        server model parameters. This allows parties to calculate their partial derivatives
        for local model updates for multiple local iterations. However, if the number of local
        iterations $Q=1$, then a more efficient method of communication
        is for the server to compute partial derivative updates for the parties~\citep{FDML,SplitNN,pyvertical},
        avoiding the need for parties to receive embeddings from other parties. 
        This approach can be applied to \mbox{C-VFL} as well.

\begin{algorithm}[H]
    \begin{algorithmic}[1]
        \rev{
    \STATE {\textbf{Initialize:}} $\theta_m^0$ for all parties $m$ and server model $\theta_0^0$
    \FOR {$t \leftarrow 0, \ldots, T-1$}
        \STATE Randomly sample $\B^t \in \{\X, \y\}$
        \FOR {$m \leftarrow 1, \ldots, M$ in parallel}
            \STATE Send $\mathcal{C}_m(h_m(\theta_m^t ; \X^{\B^t}_m))$ to server %
        \ENDFOR
        \STATE $\hat{\Phi}^t \leftarrow 
            \{\mathcal{C}_0(\theta_0), \mathcal{C}_1(h_1(\theta_1^t)), 
            \ldots, \mathcal{C}_M(h_M(\theta_M^t))\}$
        \STATE $\theta_0^{t+1} = \theta_0^t - \eta^t \nabla_0 F_{\B}(\hat{\Phi}^t ; \y^{\B^t})$
        \STATE Server sends $\nabla_{h_m(\theta_m^t ; \X^{\B^t}_m)} F_{\B}(\hat{\Phi}^t ; \y^{\B^t})$ 
        to each party $m$
        \FOR {$m \leftarrow 1, \ldots, M$ in parallel}
            \STATE $\nabla_m F_{\B}(\hat{\Phi}^t ; \y^{\B^t}) = \nabla_{\theta_m} h_m(\theta_m^t ; \X^{\B^t}_m) \nabla_{h_m(\theta_m^t ; \X^{\B^t}_m)} F_{\B}(\hat{\Phi}^t ; \y^{\B^t})$
            \STATE $\theta_m^{t+1} = \theta_m^t - \eta^t \nabla_m F_{\B}(\hat{\Phi}^t ; \y^{\B^t})$
        \ENDFOR
        \ENDFOR}
    \end{algorithmic}
    \caption{Compressed Vertical Federated Learning for $Q=1$}
    \label{cvflopt.alg}
\end{algorithm}

        The pseudo-code for this method is presented in Algorithm~\ref{cvflopt.alg}.
        In this version of \mbox{C-VFL}, parties send their compressed embeddings to the server.
        The server calculates the loss by feeding the embeddings through the server model.
        The server calculates the gradient with respect to the loss.
        The server then sends to each party $m$ the partial derivative with respect to its embedding:
        $\nabla_{h_m(\theta_m^t ; \X^{\B^t}_m)} F_{\B}(\hat{\Phi}^t ; \y^{\B^t})$.
        Each party $m$ calculates the following partial derivative with respect to its local model parameters:
        \begin{align}
            \nabla_m F_{\B}(\hat{\Phi}^t ; \y^{\B^t}) = \nabla_{\theta_m} h_m(\theta_m^t ; \X^{\B^t}_m) \nabla_{h_m(\theta_m^t ; \X^{\B^t}_m)} F_{\B}(\hat{\Phi}^t ; \y^{\B^t}).
        \end{align}
        Using this partial derivative, the party updates its local model:
        \begin{align}
            \theta_m^{t+1} = \theta_m^t - \eta^{t_0}
            \nabla_m F_{\B}(\hat{\Phi}_m^t ; \y^{\B^{t_0}}).
        \end{align}
        Note that this process is mathematically equivalent to \mbox{C-VFL} when $Q=1$; thus the analysis
        in Section~\ref{main.sec} holds. 
        The communication cost of Algorithm~\ref{cvflopt.alg} per communication round without compression is 
        $O(B \cdot \sum_m P_m)$, a reduction in communication compared to 
        the communication cost per round of Algorithm~\ref{cvfl.alg}: $O(M \cdot (B \cdot \sum_m P_m + |\theta_0|))$.
        Although Algorithm~\ref{cvflopt.alg} reduces communication in a given round, 
        it is limited to the case when $Q=1$. 
        For $Q>1$, we must use Algorithm~\ref{cvfl.alg}.

\begin{table}
\caption{MIMIC-III communication cost to reach a target test $F_1$-Score of $0.4$.
Value shown is the mean of $5$ runs, $\pm$ the standard deviation.
The first row has no embedding compression, while the second row employs 
vector quantization on embeddings with $b=3$. 
For the cases where $Q=1$, Algorithm~\ref{cvflopt.alg} is used,
and for cases where $Q>1$, Algorithm~\ref{cvfl.alg} is used.
In these experiments, the number of clients $M=4$.}
\label{optQ.table}
\vskip 0.1in
\small
\centering
\begin{tabular}{lccc}
    \toprule 
    \textbf{Compressor}  &  \multicolumn{3}{c}{\textbf{Cost (MB) to Reach Target $F_1$-Score $0.4$}} \\
                         & \subhead{$Q=1$}        &\subhead{$Q=10$}        & \subhead{$Q=25$}\\
    \midrule
    \midrule
    None $b=32$          & 4517.59 $\pm$ 465.70 & 3777.21 $\pm$ 522.43 & 1536.69 $\pm$ 201.99 \\
    \midrule
    Vector $b=3$         & 433.28 $\pm$ 79.83   & 330.47 $\pm$ 10.63   & 125.09 $\pm$ 4.56 \\
    \bottomrule                                                                               
\end{tabular}
\end{table}

        We run experiments on the MIMIC-III dataset to compare
        \mbox{C-VFL} using Algorithm~\ref{cvflopt.alg} with \mbox{C-VFL} using Algorithm~\ref{cvfl.alg} 
        with values of $Q>1$.
        We show the results of these experiments in Table~\ref{optQ.table}.
        Here, we show the communication cost to reach a target $F_1$-Score of $0.4$.
        The results in the column labeled $Q=1$ are from running Algorithm~\ref{cvflopt.alg},
        while all other results are from running Algorithm~\ref{cvfl.alg}.
        We include results for the case where \mbox{C-VFL} is run without embedding compression,
        as well as results for the case where vector quantization with $b=3$ is used to compress embeddings.
        We can see that for this dataset, values of $Q>1$ reduce the cost of communication
        to reach the target $F_1$-Score.
        In all cases, Algorithm~\ref{cvfl.alg} achieves lower communication cost to 
        reach a target model accuracy than Algorithm~\ref{cvflopt.alg},
        despite Algorithm~\ref{cvflopt.alg} having a lower communication cost
        per communication round than Algorithm~\ref{cvfl.alg}. 
        The use of multiple local iterations in Algorithm~\ref{cvfl.alg} 
        decreased the number of global rounds required to attain the target accuracy
        compared to Algorithm~\ref{cvflopt.alg}.
    }

\end{document}